\documentclass[sigconf]{acmart}

\pdfoutput=1

\usepackage{booktabs} 
\usepackage{balance} 
\usepackage{appendix}

\usepackage{hyperref}       
\usepackage{url}            
\usepackage{booktabs}       
\usepackage{amsfonts}       
\usepackage{nicefrac}       
\usepackage{microtype}      
\usepackage{tikz} 
\usepackage{pgfplots} 
\usepackage{algorithm}
\usepackage{romannum} 
\usepackage{multirow} 
\usepackage{cprotect} 
\usepackage{stackrel}
\usepackage{amsmath} 
\usepackage{amsfonts} 
\usepackage{bm} 
\usepackage{bbm} 
\usepackage{mathtools} 
\usepackage{diagbox} 
\usepackage{tabulary} 
\usepackage{rotating} 
\usepackage{subcaption} 
\usepackage{amsthm}
\makeatletter
\@ifundefined{proposition}{%
    
}{}
\@ifundefined{definition}{%
    \newtheorem{definition}{Definition}%
}{}
\@ifundefined{lemma}{%
    \newtheorem{lemma}{Lemma}
}{}
\@ifundefined{corollary}{%
    
}{}
\@ifundefined{theorem}{%
    
}{}
\@ifundefined{assumption}{%
    
}{}
\@ifundefined{problem}{%
    
}{}
\@ifundefined{claim}{%
    
}{}
\makeatother

\providecommand\JW[1]{$\spadesuit$\footnote{JW: #1}}

\usetikzlibrary{matrix} 
\usetikzlibrary{arrows} 
\usetikzlibrary{calc} 
\usetikzlibrary{fit,positioning, shapes.misc, backgrounds}
\usetikzlibrary{patterns}

\DeclareMathOperator{\sign}{sign}

\newtheorem{innercustomgeneric}{\customgenericname}
\providecommand{\customgenericname}{}
\newcommand{\newcustomtheorem}[2]{%
  \newenvironment{#1}[1]
  {%
   \renewcommand\customgenericname{#2}%
   \renewcommand\theinnercustomgeneric{##1}%
   \innercustomgeneric
  }
  {\endinnercustomgeneric}
}
\newcustomtheorem{customclaim}{Claim}

\newtheorem*{problem*}{Problem}
\newtheorem*{definition*}{Definition}

\providecommand{\Fig}{Figure~}
\providecommand{\Sec}{Section~}

\providecommand{\Def}{Definition~}
\providecommand{\Thm}{Theorem~}
\providecommand{\Pro}{Proposition~}
\providecommand{\Lem}{Lemma~}

\providecommand{\eqnref}			[1]		{Equation~(\ref{#1})}
\providecommand{\tabref}			[1]		{Table~\ref{#1}}
\providecommand{\figref}			[1]		{Figure~\ref{#1}}

\providecommand{\secref}			[1]		{Section~\ref{#1}}

\providecommand{\eqa}				[1]		{\begin{align}#1\end{align}}
\providecommand{\eqas}			[1]		{\begin{align*}#1\end{align*}}

\providecommand{\mat}				[2]		{\left[\begin{array}{#1} #2 \end{array}\right]}

\providecommand{\ie}{\emph{i.e.,}~}
\providecommand{\eg}{\emph{e.g.,}~}

\providecommand{\realnum}					{\mathbb{R}}
\providecommand{\integernum}				{\mathbb{Z}}

\providecommand{\naturalnum}				{\mathbb{N}}

\renewcommand{\(}						{\left(}
\renewcommand{\)}						{\right)}
\renewcommand{\[}						{\left[}
\renewcommand{\]}						{\right]}

\providecommand{\Exp}{\mathbbm{E}}

\def\e{\bm{e}}

\def\h{\bm{h}}

\def\m{\bm{m}}

\def\w{\bm{w}}
\def\x{\bm{x}}

\def\z{\bm{z}}

\def\N{\bm{N}}

\def\W{\bm{W}}

\def\RH{\hat{\bm{R}}}

\def\hh{\hat{{h}}}

\def\Dh{\hat{{D}}}

\def\Rh{\hat{{R}}}

\def\As{\mathcal{{A}}}
\def\Bs{\mathcal{{B}}}

\def\Dsh{\hat{\mathcal{{D}}}}

\def\Fs{\mathcal{{F}}}
\def\Gs{\mathcal{{G}}}

\def\Hs{\mathcal{{H}}}

\def\Ls{\mathcal{{L}}}
\def\Ms{\mathcal{{M}}}

\def\Ps{\mathcal{{P}}}
\def\Qs{\mathcal{{Q}}}

\def\Ss{\mathcal{{S}}}

\def\Xs{\mathcal{{X}}}
\def\Ys{\mathcal{{Y}}}

\makeatletter
\newcommand{\manuallabel}[2]{\def\@currentlabel{#2}\label{#1}}
\makeatother

\theoremstyle{acmplain}
\newtheorem{mytheorem}{Theorem}
\newtheorem{myproposition}{Proposition}
\theoremstyle{acmdefinition}
\newtheorem{mydefinition}{Definition}

\usepackage{array}
\newcolumntype{M}[1]{>{\centering\arraybackslash}m{#1}}

\newcommand{\lzo}{{\ell_{01}}}
\newcommand{\zo}{{$0$-$1$~}}

\newcommand{\Dhp}{\Dh^+}
\newcommand{\Dhn}{\Dh^-}

\newcommand{\Dha}{\Dh_\alpha}
\newcommand{\Dhap}{\Dh_\alpha^+}
\newcommand{\Dhan}{\Dh_\alpha^-}

\newcommand{\Rhl}{\Rh_\ell}

\newcommand{\Rb}{\bar{R}}

\newcommand{\hha}{\hh_\alpha}

\newcommand{\tma}{BFA}

\newcommand{\Romnum}[1]{\uppercase\expandafter{\romannumeral#1}}

\newcommand\qcomment[1]{ }

\newcommand{\cmt}[1]{{\color{purple}#1}}

\newlength{\hatchspread}
\newlength{\hatchthickness}
\newlength{\hatchshift}
\newcommand{\hatchcolor}{}
\tikzset{hatchspread/.code={\setlength{\hatchspread}{#1}},
         hatchthickness/.code={\setlength{\hatchthickness}{#1}},
         hatchshift/.code={\setlength{\hatchshift}{#1}},
         hatchcolor/.code={\renewcommand{\hatchcolor}{#1}}}
\tikzset{hatchspread=3pt,
         hatchthickness=0.4pt,
         hatchshift=0pt,
         hatchcolor=black}
\pgfdeclarepatternformonly[\hatchspread,\hatchthickness,\hatchshift,\hatchcolor]
   {custom north west lines}
   {\pgfqpoint{\dimexpr-2\hatchthickness}{\dimexpr-2\hatchthickness}}
   {\pgfqpoint{\dimexpr\hatchspread+2\hatchthickness}{\dimexpr\hatchspread+2\hatchthickness}}
   {\pgfqpoint{\dimexpr\hatchspread}{\dimexpr\hatchspread}}
   {
    \pgfsetlinewidth{\hatchthickness}
    \pgfpathmoveto{\pgfqpoint{0pt}{\dimexpr\hatchspread+\hatchshift}}
    \pgfpathlineto{\pgfqpoint{\dimexpr\hatchspread+0.15pt+\hatchshift}{-0.15pt}}
    \ifdim \hatchshift > 0pt
      \pgfpathmoveto{\pgfqpoint{0pt}{\hatchshift}}
      \pgfpathlineto{\pgfqpoint{\dimexpr0.15pt+\hatchshift}{-0.15pt}}
    \fi
    \pgfsetstrokecolor{\hatchcolor}
    \pgfusepath{stroke}
   }
\pgfdeclarepatternformonly[\hatchspread,\hatchthickness,\hatchshift,\hatchcolor]
   {custom north east lines}
   {\pgfqpoint{\dimexpr-2\hatchthickness}{\dimexpr-2\hatchthickness}}
   {\pgfqpoint{\dimexpr\hatchspread+2\hatchthickness}{\dimexpr\hatchspread+2\hatchthickness}}
   {\pgfqpoint{\dimexpr\hatchspread}{\dimexpr\hatchspread}}
   {
    \pgfsetlinewidth{\hatchthickness}
    \pgfpathmoveto{\pgfqpoint{\dimexpr\hatchshift-0.15pt}{-0.15pt}}
    \pgfpathlineto{\pgfqpoint{\dimexpr\hatchspread+0.15pt}{\dimexpr\hatchspread-\hatchshift+0.15pt}}
    \ifdim \hatchshift > 0pt
      \pgfpathmoveto{\pgfqpoint{-0.15pt}{\dimexpr\hatchspread-\hatchshift-0.15pt}}
      \pgfpathlineto{\pgfqpoint{\dimexpr\hatchshift+0.15pt}{\dimexpr\hatchspread+0.15pt}}
    \fi
    \pgfsetstrokecolor{\hatchcolor}
    \pgfusepath{stroke}
   }

\usepackage{url}
\makeatletter
\g@addto@macro{\UrlBreaks}{\UrlOrds}
\makeatother

\acmYear{2017}
\copyrightyear{2017} 
\setcopyright{none}
\acmConference[ICCPS]{The 8th ACM/IEEE International Conference on Cyber-Physical Systems}{April 2017}{Pittsburgh, PA USA}
\acmDOI{http://dx.doi.org/10.1145/3055004.3055006}
\acmPrice{\$15.00}

\begin{document}
\title[Resilient Linear Classification]{Resilient Linear Classification: An Approach to Deal with Attacks on Training Data}

\author{Sangdon Park}
\affiliation{%
  \department{Department of Computer \& Information Science}
  \institution{University of Pennsylvania}
}
\email{sangdonp@cis.upenn.edu}

\author{James Weimer}
\affiliation{%
  \department{Department of Computer \& Information Science}
  \institution{University of Pennsylvania}
}
\email{weimerj@cis.upenn.edu}

\author{Insup Lee}
\affiliation{%
  \department{Department of Computer \& Information Science}
  \institution{University of Pennsylvania}
}
\email{lee@cis.upenn.edu}

%

%
%
%
%



\begin{abstract}
Data-driven techniques are used in cyber-physical systems (CPS) for controlling autonomous vehicles, handling demand responses for energy management, and modeling human physiology for medical devices.
These data-driven techniques extract models from training data, where their performance is often analyzed with respect to random errors in the training data.
However, if the training data is maliciously altered by attackers, the effect of these attacks on the learning algorithms underpinning data-driven CPS have yet to be considered. In this paper, we analyze the resilience of classification algorithms to training data attacks. 
Specifically, a generic metric is proposed that is tailored to measure resilience of classification algorithms with respect to worst-case tampering of the training data. 
Using the metric, we show that traditional linear classification algorithms are resilient under restricted conditions.
To overcome these limitations, we propose a linear classification algorithm with a majority constraint and prove that it is strictly more resilient than the traditional algorithms.
Evaluations on both synthetic data and a real-world retrospective arrhythmia medical case-study show that the traditional algorithms are vulnerable to tampered training data, whereas the proposed algorithm is more resilient (as measured by worst-case tampering).

\end{abstract}


%
%

%



\maketitle




\section{Introduction}
\label{sec:intro}


The penetration of data-driven techniques (\eg machine learning) to monitor and control a broad range of cyber-phy\-sical systems has sharply increased. 
Autonomous cars rely on visual object detectors learned from image data for recognizing objects\cite{chen2015deepdriving, hadsell2009learning,krizhevsky2012imagenet}. 
Building demand response can be effectively handled by data-driven modeling and prediction of the electric usage of buildings \cite{behl_ICCPS16}. 
Smart insulin pumps can adapt to type 1 diabetic patients using data-driven modeling of user-specific eating and pump-using behavior \cite{chen2015data}.
While data-driven CPS offer remarkable capabilities for enhanced performance, they also introduce unprecedented security vulnerabilities with the risk of malicious attacks having catastrophic consequences.  Specifically, the training data used for learning (be it online or offline), is vulnerable to malicious tampering that can result in data-driven CPS reacting incorrectly to safety-critical events. 

The training data for data-driven CPS can be tampered in several ways, depending on the application.
In modern automobiles, multiple vulnerabilities have been demonstrated where hackers obtain full control of automobiles by eavesdropping a Controller Area Network (CAN) and injecting CAN messages \cite{checkoway2011comprehensive,koscher2010experimental}, which provides possibilities to inject malicious data being used for online learning algorithms \cite{chen2015deepdriving,hadsell2009learning}.
Furthermore, automobiles and robots, which rely on sensor inputs from global positioning system (GPS), inertial measurement unit (IMU) or wheel speed sensors, can be susceptible on spoofing attacks \cite{humphreys2008assessing,shoukry2013non,son2015rocking}. This means attackers can tamper training data collected from sensors. 
Hacking incidents on medical devices and hospitals \cite{medstarhacked2016,jjinsulinpumpvulnerable,Methodisthospitalhacekd2016} suggest attackers can tamper both device-level and data center-level training data.
Moreover, attackers with knowledge of the underlying machine learning techniques -- \emph{e.g.,} support vector machines (SVMs), principal component analysis, logistic regression, artificial neural network, and (ensemble) decision trees -- can strategically alter the training data to minimize the accuracy of the algorithms \cite{biggio2013evasion,biggio2012poisoning,goodfellow2014explaining,kantchelian2015evasion,mei2015using,Szegedy2014}, to maliciously affect the performance of data-driven CPS \cite{chen2015deepdriving,chen2015data,hadsell2009learning,behl_ICCPS16,paridari2016cyber,seo2014predicting,valenzuela2013real}.

Capabilities provided by traditional 
\emph{cyber defenses} (\emph{e.g.,} communication channel encryption and authentication), 
\emph{fault tolerant techniques} (\emph{e.g.,} data sanitization \cite{cretu2008casting}, robust loss functions \cite{wu2012robust,zhang2004solving}, and robust learning~\cite{chen2013robust, feng2014robust}), and 
\emph{adversarial learning} \cite{bruckner2011stackelberg,dalvi2004adversarial,feige2015learning} are necessary to secure data-driven CPS, but they are not sufficient. 
Specifically, the \emph{cyber defenses} are insufficient for defending against cyber-phy\-sical attacks (\eg GPS spoofing \cite{humphreys2008assessing}) where a sensing environment can be maliciously altered such that correctly functioning sensors and systems can act erroneously. 
These challenges are compounded in dynamic applications (e.g., autonomous driving and closed-loop physiological control) wh\-ere accurate physical models, commonly required for \emph{fault tolerant systems}, are challenging to obtain. 
Moreover, \emph{adversarial learning} literature (\eg \cite{bruckner2011stackelberg,dalvi2004adversarial,feige2015learning}) usually assumes a known attacker behavior and/or goal -- which is likely unknown in complex CPS applications.
Due to the shortcomings of traditional approaches for securing the training data of data-driven CPS, it is necessary to consider techniques for \emph{resilient machine learning} that can defend against cyber-physical attacks and make minimal assumptions on environments and attackers.

Towards the ultimate goal of attack-resilient machine learning, we propose a resilience metric for the analysis and design of learning algorithms under cyber-physical attacks. 
The metric aims to quantify the resilience of learning algorithms for analysis, which in turn contributes to designing resilient learning algorithms.
Specifically, this work considers binary linear classification algorithms in the presence of maliciously tampered training data. 
Binary linear classification represents a basic building block for more complex classification approaches, such as neural network, decision trees, and boosting; thus, developing attack resilient linear classifiers can lead to more advance resilient machine learning algorithms.  
To analyze binary classifiers in the presence of training data attacks, we introduce a generic measure of resilience for classification in terms of worst-case errors. 
Based on the resilience metric, traditional linear classification algorithms are evaluated. First, we prove the maximal resilience of any linear classification algorithm, which provides an upper bound of a resilience condition that can be achievable. Then, we prove that convex loss linear classification algorithms, such as SVMs, and \zo loss linear classification algorithm can not achieve maximal resilience. Based on these results, we introduce a majority \zo loss linear classification algorithm that is strictly more resilient than the traditional approaches and achieves the maximal resilience condition.

Finally, we evaluate the different classification algorithms, in the presence of attacks, on a synthetic dataset and a medical case-study, introduced in \cite{guvenir1997supervised}, to design a detector for arrhythmia (\ie irregular heart beat). The evaluation on synthetic data illustrates conditions when the different algorithms are (and are not) resilient, while the arrhythmia dataset serves to illustrate resilient binary classification in a real-world data-driven medical CPS (described in Section~\ref{sec:exp}).

In summary, the contributions of this work include: 
(i) introducing, to our knowledge, the first assessment metric for analyzing binary classifier resilience; 
(ii) providing an analysis of the resilience of traditional binary classification techniques illustrating their shortcomings;
(iii) describing a resilient classification approach that provides maximal resilience;
(iv) evaluating in a retrospective real-world arrhythmia classification case-study.

The following section describes the work most closely related to the resilient classification problem considered herein. 
In \Sec \ref{sec:setup}, we define attacker capabilities and a resilience metric.
In \Sec \ref{sec:problem}, the resilient classification problem is formally defined while an analysis of traditional linear classification algorithms is provided in \Sec \ref{sec:feasibility}.
In \Sec \ref{sec:proposed}, a new resilient linear classification algorithm is proposed which achieves maximal resilience for the attacker's capabilities considered.
Section \ref{sec:exp} illustrates the theoretical results using case studies on synthetic and medical data.  The final section provides conclusions with discussion about countermeasures and future work.

\vspace{-1ex}
\section{Related Work} \label{sec:related}

This section presents the related works for CPS security (\Sec \ref{sec:cps_security}) and traditional error/attack models in the machine learning literature (\Sec \ref{sec:related_PAC_errors}).

\vspace{-1.5ex}
\subsection{CPS security} \label{sec:cps_security}
\vspace{-0.5ex}
Though the security of learning systems for data-driven CPS has been an afterthought, the security of CPS has seen much effort in the past decade. 
A mathematical framework considering attacks on CPS is proposed in \cite{cardenas2008research,pasqualetti2013attack}.
The necessary and sufficient conditions on CPS with a failure detector such that a stealthy attacker can destabilize the system are provided in~ \cite{mo2010false}.
State estimation for an electric power system is analyzed in~\cite{teixeira2010cyber} assuming attackers know a partial model of the true system.
Resilient state estimators for CPS that tolerate a bounded number of sensors and/or actuators attacks are considered in \cite{fawzi2014secure,pajic2014robustness}.  
In mobility-as-a-Service systems (\eg ride-sharing services), it has been demonstrated that a fraction of cars are maliciously called by fake reservation for denial-of-service \cite{yuan2016zubers}. 
Surgical tele-operated robotic systems can be affected by denial-of-service attacks on communication channels \cite{bonaci2015experimental}. 
Energy management systems, especially when connected to building networks, are vulnerable to cyber attacks that impact on the systems operation. This vulnerability can be attenuated by applying resilient policy when attacks are detected \cite{paridari2016cyber}. While there has been much recent work on CPS security, these approaches are (in general) not directly applicable to data-driven CPS.


%

\vspace{-1.5ex}
\subsection{Learning with Errors} \label{sec:related_PAC_errors}
\vspace{-0.5ex}

In this subsection, we review the literature on learning in the presence of training data errors most closely related to our work, where a more complete survey of the entire literature can be found in \cite{goldman1995can,natarajan2013learning}. The error models can be categorized as either {label errors} or {feature errors} in \tabref{tab:errors}, according to their classical definitions \cite{goldman1995can,natarajan2013learning}. 
%
 %
Under each error model, the performance of a learning algorithm is analyzed against whether it achieves a desired classifier.

\vspace{-1ex}
\begin{table}[htb!]
\centering
\begin{tabular}{c||cc}
\toprule
\multirow{2}{*}{label errors} 
& class-independent (CICE) & \cite{angluin1988learning} \\ 
& class-dependent (CDCE) & \cite{natarajan2013learning} \\ \midrule
\multirow{3}{*}{feature errors} 
& uniform random (URAE) & \cite{sloan1988types} \\ 
& product random (PRAE) & \cite{goldman1995can} \\ 
& malicious errors (ME) & \cite{kearns1993learning} \\ 
\bottomrule
\end{tabular}
\caption{Taxonomy of training data errors in the literature.}
\label{tab:errors}
\end{table}
\vspace{-5ex}

When labels in training data are corrupted, the training data is said to have {label errors}, which can be divided into two subtypes: {class-independent classification errors} (CICE) \cite{angluin1988learning} and {class-dependent classification errors} (CDCE) \cite{natarajan2013learning}. 
The class-independent classification error model assumes the error probability of positive and negative labels are same while the class-dependent classification error model allows the different error probability for positive and negative labels. 

When features in the training data are corrupted, the training data is said to have {feature errors}, which can be divided into three subtypes: {uniform random attribute errors} \cite{sloan1988types}, {product random attribute errors} \cite{goldman1995can}, and {malicious errors} \cite{kearns1993learning}. 
Both the uniform random attribute error (URAE) and the product random attribute error (PRAE) models assume errors on features (\ie columns of the feature matrix), where URAE assumes the same error probability for all features and PRAE allows for variable error probabilities.  From a CPS perspective, attacks on individual features require that each column of the feature matrix corresponds to a single attack surface (\eg a single sensor) -- which restricts the use of multiple sensors in a single feature, as common in data-driven CPS~\cite{chen2015deepdriving, hadsell2009learning}. 
Different from URAE and PRAE, the malicious error (ME) model assumes arbitrary attacks on feature vectors (\ie rows of the feature matrix). However, the ME model assumes the probabilities of attacking the feature vectors corresponding to positive and negative labels are the same -- a condition which may not be satisfied by savvy attackers. In contrast to this, our error (or attack) model assumes the probabilities can be different. 
\vspace{-1ex}
\section{Setup for Resilient Binary Classification} \label{sec:setup}


This section introduces essential definitions that are the bases for describing resilient binary classification problem. In the following subsections, we present a traditional binary linear classification problem (\Sec \ref{sec:bin-classification}), define our attacker assumptions (\Sec \ref{sec:attack-assumptions}), and introduce a resilience metric (\Sec \ref{sec:resilience-metric}).

%
%
%
%
Notationally, we write 
$\realnum$, $\realnum_0^+$, $\naturalnum_{0}$ and $[a, b]$ to denote the set of real numbers, non-negative real numbers, non-negative integers, and integers from $a$ to $b$, respectively.  We write $\mathbbm{1}$ as the ones vector of an appropriate size and
$| \cdot |$ to denote the cardinality (\ie number of elements) of a finite set.
The sign function is written as $\sign : \realnum \rightarrow \{+1,-1\}$ and 
$\mathbbm{1}\{\cdot\}$ corresponds to the indicator function that maps true and false to $1$ and $0$.  
Additionally, we write $\ell_{01}$ to denote a $0$-$1$ loss function, such that $\ell_{01}(y_i, h(\x_i)) = \mathbbm{1}\left\{y_i \neq \sign(h(\x_i))\right\}$.
Lastly, $\| \cdot \|_1$ and $\| \cdot \|_\infty$ denotes the $1$-norm and the $\infty$-norm, respectively. See \tabref{tab:glossary} for the glossary of mathematical notations in this paper.

\begin{table}[tb!]
\centering
\begin{tabular}{c||p{45ex}}
\toprule
symbol & description \\ \hline \hline
$\Dh$ & actual training data \\ \hline
$\Dsh$ & class of training data \\ \hline
$\Dhp$ & positive training data  \\ \hline
$\Dhn$ & negative training data \\ \hline
$\Theta$ & set of attacker capability parameters \\ \hline
$\alpha$ & attacker capability parameter where $\alpha \in \Theta$  \\ \hline
$\Dha$ & tapered training data \\ \hline
$\Dsh_{\alpha}$ & class of tampered training data \\ \hline
$\Dhap$ & positive tampered training data  \\ \hline
$\Dhan$ & negative tampered training data \\ \hline
$N$ & number of training data pairs (\ie $|\Dh|$ ) \\ \hline
$\N$ & pair of $|\Dhp|$ and $|\Dhn|$ \\ \hline
$\Fs$ & set of classifiers \\ \hline
$\Hs$ & subset of classifiers (\ie $\Hs \subseteq \Fs$) \\ \hline
$\Ls$ & set of linear classifiers \\ \hline
$\Ss$ & set of loss functions \\ \hline
$\ell$ & loss function in $\Ss$ \\ \hline
$\ell_{c}$ & convex loss function in $\Ss$ \\ \hline
$\ell_{01}$ & 0-1 loss function in $\Ss$ \\ \hline
$P$ & classification algorithm \\ \hline
$\Ps_{\Fs, \Ss}$ & class of classification algorithms \\ \hline
$P_{\Hs, \ell}$ & classification algorithm \\ \hline
$\Ps_{\Ls, \Ss}$ & class of linear classification algorithms \\ \hline
$P_{\Ls, \ell_{c}}$ & convex loss linear classification algorithm \\ \hline
$P_{\Ls, \ell_{01}}$ & \zo loss linear classification algorithm \\ \hline
$P_{\Ms, \ell_{01}}$ & majority \zo loss linear classification algorithm \\ \hline
$g_{P}$ & resilience bound of a classification algorithm $P$\\ \hline
$\Gs$ & set of resilience bounds \\ \hline
$\As_{P}$ & resilience attack condition of a classification algorithm $P$ where $\As_{P} \subseteq \Theta$ \\ \hline
$\Bs_{P}$ & perfectly attackable condition of a classification algorithm $P$ where $\Bs_{P} \subseteq \Theta$ \\
\bottomrule
\end{tabular}
\caption{The glossary of mathematical notations.}
\vspace{-8ex}
\label{tab:glossary}
\end{table}

\vspace{-1.5ex}
\subsection{Traditional Binary Classification}
\label{sec:bin-classification}
\vspace{-0.5ex}
We begin by considering the traditional problem of binary classification in the absence of attacks (or errors). Namely, we consider un-attacked training data $\Dh\! =\! \{(\x_i, y_i)\}_{i=1}^N \in \Dsh$, where 
$N$ is the number of training data pairs, 
$\Dsh = \{ D \subset \Xs \times \Ys $$ \; | \; |D| = N \}$ is a class of training data with $N$ pairs, 
$\Xs \subseteq \realnum^p$ corresponds to a set of feature vectors (or attributes), 
$\Ys = \{-1, +1\}$ denotes the set of labels (or classes), and
each element of $\x_i$ is called a {feature}.
In a traditional (binary) classification problem, such as \cite{ vapnik1999overview}, given training data, a designer specifies a set of (real-valued) classifiers $\Hs \subseteq \realnum^{\Xs}$, and a loss function $\ell : \Ys \times \realnum \rightarrow \realnum_0^+$, to learn a (real-valued) classifier $\hh \in \Hs$, according to
%
\vspace{-1ex}
\eqa{
P_{\Hs, \ell}(\Dh) : \hat{h} = \arg \min_{h \in \mathcal{H}} \| \W \cdot \RH_\ell(h | \Dh) \|_1 \label{eq:classifier},
}
%
where $\W \in \realnum^{2\times2}$ is the diagonal matrix with the positive risk weight $w^+$ and the negative risk weight $w^-$ on the diagonal, and zeros elsewhere. $\RH_\ell(h|\Dh) \in \realnum^2$ denotes the bi-dimensional vector of empirical risks corresponding to the positive and negative training data.  Specifically, we write $\RH_\ell(h|\Dh)\!=\!\mat{cc}{\!\!\Rh_\ell(h|\Dhp) \!&\! \Rh_\ell(h|\Dhn)\!\!}^\top$, where  
$\Rhl(h | \Dh) =  \frac{1}{|\Dh|} \sum_{i=1}^{|\Dh|} \ell (y_i, h(\x_i))$ is the normalized empirical risk evaluated over the training data and 
$\Dhp$ and $\Dhn$ corresponds to the mutually exclusive sets of positive and negative training data pairs, respectively, such that $\Dh = \Dhp \cup \Dhn$.
We note that we use \eqnref{eq:classifier} for distinguishing empirical risks over positive and negative training data, but it is equivalent to the standard notation \cite{vapnik1999overview} if $w^+ = |\Dhp|$ and $w^- = |\Dhn|$, and we assume the standard notion in this paper.
Also, we call $\hh$ a classifier (\ie $\hh \in \Ys^{\Xs}$) or a real-valued classifier (\ie $\hh' \in \realnum^{\Xs}$), interchangeably, assuming the composition of a sign function and a real-valued classifier (\ie $\sign \circ \hh' \in \Ys^{\Xs}$) is a classifier. 
Moreover, we say $N$ is the number of training data pairs or $N = (|\Dhp|, |\Dhn|)$, interchangeably.

In this paper, we consider a set of classification algorithms $\Ps_{\Fs, \Ss}$, where 
$\Fs$ is a set of classifiers and 
$\Ss$ is the set of monotonically non-increasing functions that are lower-bounded by a \zo loss function. Specifically, the loss function $\ell(y, h(\x))$ is represented as $\ell(y, h(\x)) \allowbreak = \phi(t)$, where 
$t = y h(\x)$,
$\phi$ is lower-bounded by $\mathbbm{1}\left\{ t \leq 0 \right\}$, 
$\phi(0) = 1$,
$\phi$ is a monotonically non-increasing function, and 
$\lim_{t \to \infty} \phi(t) = c$ for some scalar $c < 1$. 
We note that these assumptions generalize a convex loss \cite{bartlett2006convexity} to cover a non-convex loss.

Each algorithm in $\Ps_{\Fs, \Ss}$ is a map from a class of training data $\Dsh$ to a subset of $\Fs$ that uses a loss function in $\Ss$ (\ie $\Ps_{\Fs, \Ss} \supseteq \{ P_{F, \ell} | F \subseteq \Fs, \ell \in \Ss \}$).
Thus, empirical risk minimization (\eqnref{eq:classifier}) for any hypothesis space $\Hs \subseteq \Fs$ and a loss function $\ell \in \Ss$ is also a classification algorithm considered here (\ie $P_{\Hs, \ell} \in \Ps_{\Fs, \Ss}$).

\vspace{-1.5ex}
\subsection{Attacker Capabilities}
\label{sec:attack-assumptions}
\vspace{-0.5ex}

In this work, we introduce a new class of an attack based on the number of training data elements the attacker can manipulate, referenced to as a \emph{bounded feature attack} (\tma). 
Specifically, in this class of an attack, we assume the attacker has the following three capabilities;
(\romannum{1}) The attacker knows the classification algorithm to be attacked,
(\romannum{2}) the attacker has unbounded computing power,
(\romannum{3}) the attacker knows all the training data (both before and after tampering), and
(\romannum{4}) the attacker can tamper the training data. 
However, the ability to tamper the training data is limited such that the 
\emph{tampered training data} $\Dha$ differs from the original training data $\Dh$ by a finite number of elements.  We parameterize the tampered training data using an attacker capability parameter $\alpha = (\alpha^+, \alpha^-) \in \Theta = [0, |\Dhp|]\times[0, |\Dhn|]$ such that at most $\alpha^+$ and $\alpha^-$ number of positive and negative feature vectors are maliciously manipulated, respectively.
Formally, the $\alpha$-bounded feature attack is defined as follows:
\begin{mydefinition}[bounded feature attack] \label{def:threat_model}
Given 
$P_{\Hs, \ell}$,
$\Dh$, and $\alpha$, 
then $\Dha$ is a bounded feature attack (\tma) 
if $\Dha \in \Dsh$ satisfies the following two conditions:
\eqa{
\label{eq:tamper}
\text{(\romannum{1})~} | \Dhap \backslash  \Dhp | \leq \alpha^{+} , ~\text{(\romannum{2})~}  | \Dhan \backslash  \Dhn | \leq \alpha^{-}.
}
\end{mydefinition}
\noindent
Additionally, let $\Dsh_{\alpha}$ be the set of all such $\Dha$ (\ie $\Dha \in \Dsh_{\alpha} \subseteq \Dsh$).
We emphasize that \Def \ref{def:threat_model} only specifies what an attacker can do and which information can be used -- but does not indicate \emph{how} the attacker changes the data.  This definition is consistent with the attacker capability definition used in the CPS security literature (\eg \cite{fawzi2014secure,pajic2014robustness}).
Moreover, $\alpha$ is unknown in general, so algorithms considered in this paper do not assume anything on $\alpha$.

The \tma~represents a practical model of attacker capabilities. For example, assume several devices collect medical data and store it in the hospitals central data center. An attacker can exploit known vulnerabilities of the enterprise system of the data center to gain read access on all data (\ie knows all data), but can only alter data from specific devices having a certain vulnerabilities (\ie attacks some of the data). Here, we assume obtaining write access is more difficult than obtaining read access. 
\cmt{
}

%
%
%
In comparison to other attack models discussed in \Sec \ref{sec:related}, we emphasize that the proposed attacker capabilities are quite general; we only limit the number of tampered feature vectors.  By definition, the \tma~includes the ME; moreover, the \tma~can represent attacks on (maliciously) manipulating labels in training data. This is achieved by manipulating a positive feature vector into one of the negative feature vectors, which effectively switches the label from positive to negative and suggests the \tma~includes the CICE and CDCE models.

%

\vspace{-1.5ex}
\subsection{Resilience Metric}
\label{sec:resilience-metric}
\vspace{-0.5ex}

To evaluate a classification algorithm in the presence of a \tma, we aim to quantify the effect of a \tma~on the learned classifier's \emph{worst-case error metric} over all training data and all possible attacks. 
In traditional detection and classification theory, the true-positive and true-negative rates (or the corresponding false-positive and false-negative rates) are commonly used to evaluate the performance of a classifier.
We introduce a generic metric that utilizes the false-positive and false-negative rates such that 
it measures the worst-case wei\-ghted $p$-norm of the two error rates over all training data and all feasible attacks, defined as follows:
\begin{mydefinition}[resilience metric]
Given $\N$ and $\alpha$, the resilien\-ce of $P_{\Hs, \ell}$ is quantified as the worst-case weighted $p$-norm of error rates over all $\Dh \in \Dsh$ and $\Dha \in \Dsh_{\alpha}$, stated as
\eqa{ \label{eq:worst_classifier_measure}
\!\!V_{\W, p}(P_{\Hs, \ell} | \N, \alpha) \! = \! \max_{\Dh, \Dha}
\left\|
\W \! \cdot \! \RH_{\ell_{01}}(P_{\Hs, \ell}(\Dha) | \Dh) 
\right\|_{p}.
}
\end{mydefinition}
\noindent
This resilience metric measures the performance of a classification algorithm (\ie $V_{\W, p}(\cdot)$) in the presence of the worst-cast attack given the attacker capability parameter $\alpha$. 

In this work, we select $w^+ = w^- = 1$ and $p = \infty$. So, $V_{\W, p}(\cdot)$ ranges from zero to one and equals one if $P_{\Hs, \ell}$ outputs any classifier such that an attack could result in mis-classification of all the positive or negative feature vectors in the un-attacked training data $\Dh$. For notational simplicity, we denote $V_{\W, p}(\cdot)$ as $V(\cdot)$.
Our selection of $w^+ = w^- = 1$ means each label is equally important to model the unknown attacker's preference for each label.
The choice of $p = \infty$ is motivated by the worst-case classification approach that minimizes the maximum of class-conditional error rates \cite{lanckriet2002robust}.

We note that other norm measures could have been chosen rather than the $\infty$-norm.  
For instance, selecting $p = 1$ results in evaluating the $1$-norm of the false-positive and false-negative rates, where $V(\cdot) \geq 1$ implies that the classifier is at least as bad as a weighted coin-flip (\ie a trivial classifier)~\cite{van2004}. 
 Additionally, selecting $p = 2$ specifies $V(\cdot)$ to be the Euclidean distance to the classifier error of zero.  In general, the selection of $p$ in \eqnref{eq:worst_classifier_measure} can vary based upon the security concerns.

Applying the resilience metric in \eqnref{eq:worst_classifier_measure}, a binary classification algorithm $P_{\Hs, \ell}$ can be evaluated for given $\N$ and $\alpha$. Furthermore, the resilience metric can be upper bound\-ed by a function in $\N$ and $\alpha$, \ie $g(\N, \alpha): (\naturalnum_{0} \times \naturalnum_{0} \times \naturalnum_{0}) \rightarrow [0, 1]$ where $\Gs$ is the set of all such $g$. Then, the upper bound characterizes the property of an algorithm over various attack parameters. In this case,  the classification algorithm is called a $g(\N, \alpha)$-resilient algorithm. Formally, we define the resilience property of a classification algorithm in the context of this work as follows.


%
\begin{mydefinition}[${g(\N, \alpha)}$-resilience] \label{def:robustness_definition}
A classification algorithm $P$ is {${g(\N, \alpha)}$-resilient} to a \tma~if
\eqa{
V(P | \N, \alpha) \leq g(\N, \alpha),
}
where $g \in \Gs$ denotes the worst-case resilience bound.
\end{mydefinition}

\noindent
This worst-case resilience bound plays a key role in defining resilient binary classification problem, which is defined in the following section.

\vspace{-1ex}
\section{Problem Formulation} \label{sec:problem}
This section formulates the problem of analyzing (and ultimately designing) resilient binary classification algorithms with respect to training data attacks. 
Specifically, given the number of positive and negative training data $\N$, a set of classifiers $\Fs$, a set of loss functions $\Ss$, and a class of algorithms $\Ps_{\Fs, \Ss}$, the goal of this paper is finding a classification algorithm $P$ and a resilience bound $g$ that minimize the error of the resilience bound such that $P$ is $g(\N, \alpha)$-resilient to a \tma. Here, to measure the error of the resilience bound we use the number of $\alpha$ that makes the resilience bound maximum (\ie $\sum_{\alpha \in \Theta} \mathbbm{1}\{ g(\N, \alpha) = 1\}$), but any other error measure can be used.
In short, a resilient binary classification problem is defined as follows:
\begin{problem*}[\tma~resilient binary classification problem]
Gi\-ven $\N$, $\Fs$, $\Ss$, and $\Ps_{\Fs, \Ss}$,
a \tma~resilient binary classification problem is to find a classification algorithm $P \in \Ps_{\Fs, \Ss}$ and a resilience bound $g \in \Gs$ according to 
\eqa{
	(P,g) = \arg \min_{P, g} &~ \sum_{\alpha \in \Theta} \mathbbm{1}\{ g(\N, \alpha) = 1\} \\
	\text{s.t.} 
	&~ V(P | \N, \alpha) \leq g(\N, \alpha), \forall \alpha \in \Theta. \label{eq:prob-constraint1}
}
\end{problem*}
\noindent
We note several implications of the above problem.
First, a feasible classification algorithm of this problem guarantees the worst-case performance characterized by $g$ since the constraint of the problem enforces that the worst-case error (\ie $V(\cdot)$) is bounded by $g$ for all possible attacks (\ie $\forall \alpha$).
Next, the problem can consider the capabilities of classification algorithms by encoding prior knowledge on the class of classification algorithms $\Ps_{\Fs, \Ss}$. Specifically, $\Ps_{\Fs, \Ss}$ can be the class of classification algorithms that uses empirical risk minimization over $\Fs$ with convex loss functions. The resilient classification problem then finds a classification algorithm in the restricted class of $\Ps_{\Fs, \Ss}$. We note that when choosing the restricted class of classification algorithms in this paper, we do not consider the attacker capability parameter $\alpha$, implying we focus on finding an algorithm without assumptions on $\alpha$. 
Then, the ultimate goal of the resilient binary classification problem is making $g(\cdot) \leq \epsilon$ for some $\N$ and for all $\alpha$ given conditions on $\Ps_{\Fs, \Ss}$, where $\epsilon$ is a sufficiently small scalar.
Finally, we note that the resilience binary classification problem is related to the problem of minimizing generalization error of a classifier considered in traditional classification (See \Sec \ref{sec:connection_stl}).

We note that in this paper a \tma~resilient binary classification problem is simply called a resilient binary classification problem assuming a \tma~as an attack model.
$g_{P}$ denotes the optimal $g$ of the resilient binary classification problem to explicitly represent the dependency on $P$.
Also, an algorithm $A$ is \emph{more resilient} than an algorithm $B$ 
if $\sum_{\alpha \in \Theta} \mathbbm{1}\{ g_{A}(\N, \alpha) = 1\} \leq \sum_{\alpha \in \Theta} \mathbbm{1}\{ g_{B}(\N, \alpha) = 1\}$, and $A, B$, $g_{A}$, and $g_{B}$ satisfy the constraint in the problem (\eqnref{eq:prob-constraint1}).
In the following section, we utilize the definition of the resilient binary classification problem to analyze traditional linear classification algorithms for resilience under a \tma.
\vspace{-1ex}
\section{Resilience of Traditional \\ Linear Classification} \label{sec:feasibility}

Traditional classification algorithms (e.g. SVMs or 0-1 loss linear classification) rarely consider a learning environment that is partially controlled by attackers. 
Here, we focus on linear classification algorithms (\ie $\Fs = \Ls$, where $\Ls$ is the set of linear functions), which is a basic building block for more complex classification algorithms.
In this section, we analyze whether traditional linear classification algorithms are resilient. 
First, linear classification algorithms with various convex loss functions are analyzed (\Sec \ref{sec:feasibility_conv}). 
Next, a linear classification algorithm with a 0-1 loss function is analyzed (\Sec \ref{sec:feasibility_01}). 

In the following, we strictly consider un-attacked training data $\Dh$ for which a perfect classifier exists -- \ie for some $h \in \Hs$, $\| \W \cdot \RH_\ell(h | \Dh) \|_1= 0$ -- such that only errors are introduced by attacks. We note that, in practice, the empirical risk over training data is rarely equal to zero due to errors from noise and an assumption on $\Hs$. However, by treating errors as attacks, the theoretical results in the following sections can be interpreted as assuming \emph{worst-case errors} -- \eg attacks.

The resilient binary classification problem finds a classification algorithm $P$ and a resilience bound $g_{P}$, but
the resilience bound may be trivial for some $\alpha$, \ie $g(\N, \alpha) = 1$. Thus, it is worthwhile to find a \emph{resilience attack condition}, $\As_{P} \subseteq \Theta$, such that $g(\N, \alpha)$ is non-trivial for all $\alpha \in \As_{P}$ . In this case, we say that $P$ is \emph{resilient} w.r.t. $\As_{P}$.
\begin{mydefinition}[resilient w.r.t. $\As_P$]
Given $\N$, $P$, $g_{P}$, and $\As_{P}$, a classification algorithm $P$ is resilient w.r.t $\As_{P}$ if 
the algorithm is $g(\N, \alpha)$-resilient to a \tma~and $g(\N, \alpha) < 1$ for all $\alpha \in \As_{P}$.
\end{mydefinition}
Here, we emphasize that finding an attack condition on $\alpha$ that makes a classification algorithm $1$-resilient to a \tma~(\ie finding some set $\Bs_{P}$ such that $\Bs_{P} \subseteq \As_{P}^{c}$) is equally important to finding the resilience attack condition $\As_{P}$ since $\alpha \in \Bs_{P}$ can be a ``breaking point'' of the algorithm $P$.
We refer to $\Bs_{P}$ as the \emph{perfectly attackable condition} of $P$.
Thus, we introduce a new notion, \emph{perfectly attackable w.r.t $\Bs_P$}, which is formally described as follows:
\begin{mydefinition}[perfectly attackable w.r.t. $\Bs_P$]
Given $\N$, $P$, and $\Bs_{P}$,
a classification algorithm $P$ is perfectly attackable w.r.t $\Bs_{P}$ if the algorithm is $1$-resilient to a \tma~for all $\alpha \in \Bs_{P}$. 
\end{mydefinition}
\noindent
Next, we introduce a \emph{maximal resilience attack condition} $\bar{\As} \subseteq \Theta$. It is a resilience attack condition of some linear classification algorithm or a combination of algorithms where the size of the condition is maximal. Formally, 
\begin{mydefinition}[maximally resilient condition]
$\bar{\As}$ is a maximal resilient condition if
	$\bar{\As} = \cup_{P \in \Qs_{\Ls, \Ss}} \As_{P}$, where $\Qs_{\Ls, \Ss} = \{P_{\Ls, \ell} | \ell \in \Ss\} \cup P_{\Ms, \ell_{01}} \subseteq \Ps_{\Ls, \Ss}$ and $P_{\Ms, \ell_{01}}$ is defined in \secref{sec:proposed}.
\end{mydefinition}
\noindent
We note that if the resilience attack condition $\As_{P}$ of a classification algorithm $P$ is same as $\bar{\As}$, we say that $P$ is \emph{maximally resilient}.
To find the maximal resilience attack condition $\bar{\As}$, we consider some superset of it (\ie $\bar{\Bs}^{c}$ such that $\bar{\As} \subseteq \bar{\Bs}^{c}$), which is a theoretical upper bound of the maximal resilience attack condition. We argue that there exists some classification algorithm that achieves the attack condition $\bar{\Bs}^{c}$. This then implies $\bar{\Bs}^{c}$ is the maximal resilience attack condition (See \Thm \ref{thm:maximal_resilience_condition}).

One example of $\bar{\Bs}$ can be some subset of $\cap_{\ell \in \Ss} \Bs_{P_{\Ls, \ell}}$ due to 
$\bar{\As} = \cup_{\ell \in \Ss} \As_{P_{\Ls, \ell}} \subseteq \cup_{\ell \in \Ss} \Bs_{P_{\Ls, \ell}}^{c} \subseteq \bar{\Bs}^{c}$. 
The following theorems formally state $\bar{\Bs}$ and a condition when $\bar{\Bs}^{c}$ is the maximal resilience attack condition.
\begin{mytheorem} \label{thm:maximal_resilience}
Given $| \Dhp |$ and $| \Dhn |$, let $\bar{\Bs}$ be
\eqa{
\left\{ \alpha \middle| \alpha^+ \geq \frac{1}{2} |\Dhp| \text{~or~} \alpha^- \geq \frac{1}{2} |\Dhn| \right\}. \label{eq:maximual_resilience_condition}
}
For all $\ell \in \Ss$, $P_{\Ls, \ell}$ is perfectly attackable w.r.t. $\bar{\Bs}$.
\end{mytheorem} 
\vspace{-3ex}
\begin{proof}[proof sketch]
For all $\N$, $\alpha \in \bar{\Bs}$, and $\ell \in \Ss$, we find some $\Dh$ and $\Dha$ where $V(P_{\Ls, \ell}(\Dha) | \N, \alpha) = 1$.
See \Sec \ref{sec:maximal_resilience_proof} for details.
\end{proof}
\begin{mytheorem}  \label{thm:maximal_resilience_condition}
If there exists $P \in \Qs_{\Ls, \Ss}$ such that $\As_{P} = \bar{\Bs}^{c}$, then $\bar{\Bs}^{c}$ is the maximal resilience attack condition.
\end{mytheorem}
\vspace{-3ex}
\begin{proof}[proof sketch]
We use the following two set relations to prove $\bar{\As} = \bar{\Bs}^{c}$: (1) $\bar{\As} \subseteq \bar{\Bs}^{c} = \As_{P}$ and (2) $\As_{P} \subseteq \bar{\As}$.
See \Sec \ref{sec:maximal_resilience_condition_proof} for details.
\end{proof}

\noindent
The intuitive interpretation of $\bar{\Bs}$ is that 
if the number of tampered positive or negative feature vectors is greater than or equal to the half of $|\Dhp|$ or $|\Dhn|$, respectively, 
then any linear classification algorithm trained with this training data can be perfectly attackable w.r.t. $\bar{\Bs}$.
We note that in \Sec \ref{sec:proposed} we show $\bar{\Bs}$ is actually the maximal resilience attack condition. Thus, we assume this from now on.
In the following subsections, we show that two classical approaches do not achieve the maximal resilience: (i) convex loss linear classification; and (ii) 0-1 loss linear classification.

\begin{figure}[tb!]
\centering
\begin{subfigure}[c]{0.32\linewidth}
\captionsetup{justification=centering}
\begin{tikzpicture}[scale=0.3]
  \tikzset{
        >=stealth',
        help lines/.style={dashed, thick},
        axis/.style={<->},
        important line/.style={thick},
        connection/.style={thick, dotted},
        }
  \coordinate (y) at (0,5.5);
  \coordinate (x) at (5.5,0);
    
  \path
  coordinate (start) at (0,0)
  coordinate (p1) at (2.5, 5)
  coordinate (p2) at (5, 5)
  coordinate (p3) at (5/3, 2.5)
  coordinate (p4) at (5, 2.5)
  coordinate (p5) at (2.5, 2.5)
  coordinate (x1) at (2.5, 0)
  coordinate (x2) at (5, 0)
  coordinate (x3) at (5/3, 0)
  coordinate (y1) at (0, 2.5)
  coordinate (y2) at (0, 5)
  ;
  
  \begin{axis}[
        xmin=0,xmax=1,
        xlabel={{\fontsize{25}{30}$\frac{\alpha^+}{|\Dhp|}$}},
        xlabel style={font=\Huge},
        xtick={0, 0.5, 1},
        tick label style={font=\Huge},
        ymin=0,ymax=1,
        ylabel={{\fontsize{25}{30}$\frac{\alpha^-}{|\Dhn|}$}},
        ylabel style={font=\Huge, at={(-0.05,0.5)}, rotate=-90},
        ytick={0.5, 1},
        axis on top,
        ]
        legend style={legend cell align=right,legend plot pos=right}] 
        
        \addplot[fill=red, opacity=0.2, color=red, area legend] 
	coordinates {
        		(0, 0) (1, 0) (1, 1) (0, 1)
		} \closedcycle;
	\addplot[draw=black, only marks,mark=*,mark options={scale=1, fill=blue},text mark as node=true, forget plot] 
	coordinates {
		(0, 0)
	};
    \end{axis}
\end{tikzpicture}
\vspace{-4ex}
\caption{{\smaller Convex loss}}
\vspace{-2ex}
\label{fig:conv_perfectly_attackble}
\end{subfigure}
\begin{subfigure}[c]{0.32\linewidth}
\captionsetup{justification=centering}
\centering
\begin{tikzpicture}[scale=0.3]
  \tikzset{
        >=stealth',
        help lines/.style={dashed, thick},
        axis/.style={<->},
        important line/.style={thick},
        connection/.style={thick, dotted},
        }
  \coordinate (y) at (0,5.5);
  \coordinate (x) at (5.5,0);
    
  \path
  coordinate (start) at (0,0)
  coordinate (p1) at (2.5, 5)
  coordinate (p2) at (5, 5)
  coordinate (p3) at (5/3, 2.5)
  coordinate (p4) at (5, 2.5)
  coordinate (p5) at (2.5, 2.5)
  coordinate (x1) at (2.5, 0)
  coordinate (x2) at (5, 0)
  coordinate (x3) at (5/3, 0)
  coordinate (y1) at (0, 2.5)
  coordinate (y2) at (0, 5)
  ;
  
  \begin{axis}[
        xmin=0,xmax=1,
        xlabel={{\fontsize{25}{30}$\frac{\alpha^+}{|\Dhp|}$}},
        xlabel style={font=\Huge},
        xtick={0, 0.5, 1},
        ymin=0,ymax=1,
        ylabel={{\fontsize{25}{30}$\frac{\alpha^-}{|\Dhn|}$}},
        ylabel style={font=\Huge, at={(-0.05,0.5)}, rotate=-90},
        ytick={0.5, 1},
        tick label style={font=\Huge},
        axis on top,
        extra x ticks={0.33},
        extra x tick labels={{\fontsize{20}{24}$\frac{|\Dhn|}{|\Dhp|}$}},
        legend style={legend cell align=right,legend plot pos=right}]

        
        \addplot[fill=blue, opacity=0.2, area legend] 
        coordinates {
        		(0, 0) (0, 0.5) (1/6, 0.5) (1/3, 0)
        } \closedcycle;
        
        \addplot[fill=red, opacity=0.2, color=red, area legend] 
	coordinates {
		(1/3, 0) (1, 0) (1, 1) (0, 1) (0, 0.5) (1/6, 0.5) (1/3, 0)
		} \closedcycle;
	\addplot[color=black]
	coordinates {
		(1/3, 0) (1, 0) (1, 1) (0, 1) (0, 0.5) (1/6, 0.5) (1/3, 0)
	} \closedcycle;
    \end{axis}
\end{tikzpicture}
\vspace{-4ex}
\caption{{\smaller 0-1 loss}}
\vspace{-2ex}
\label{fig:01_perfectly_attackble}
\end{subfigure} 
\begin{subfigure}[c]{0.32\linewidth}
\captionsetup{justification=centering}
\centering
\begin{tikzpicture}[scale=0.3]
  \tikzset{
        >=stealth',
        help lines/.style={dashed, thick},
        axis/.style={<->},
        important line/.style={thick},
        connection/.style={thick, dotted},
        }
  \coordinate (y) at (0,5.5);
  \coordinate (x) at (5.5,0);
    
  \path
  coordinate (start) at (0,0)
  coordinate (p1) at (2.5, 5)
  coordinate (p2) at (5, 5)
  coordinate (p3) at (5/3, 2.5)
  coordinate (p4) at (5, 2.5)
  coordinate (p5) at (2.5, 2.5)
  coordinate (x1) at (2.5, 0)
  coordinate (x2) at (5, 0)
  coordinate (x3) at (5/3, 0)
  coordinate (y1) at (0, 2.5)
  coordinate (y2) at (0, 5)
  ;
  
  \begin{axis}[
        xmin=0,xmax=1,
        xlabel={{\fontsize{25}{30}$\frac{\alpha^+}{|\Dhp|}$}},
        xlabel style={font=\Huge},
        xtick={0, 0.5, 1},
        ymin=0,ymax=1,
        ylabel={{\fontsize{25}{30}$\frac{\alpha^-}{|\Dhn|}$}},
        ylabel style={font=\Huge, at={(-0.05,0.5)}, rotate=-90},
        ytick={0.5, 1},
        tick label style={font=\Huge},
        axis on top,
        legend style={legend cell align=right,legend plot pos=right}] 
        
	\addplot[fill=red, opacity=0.2, area legend] 
	coordinates {
        		(0.5, 0) (1, 0) (1, 1) (0, 1) (0, 0.5) (0.5, 0.5)
		} \closedcycle;
	\addplot[color=black]
	coordinates {
        		(0.5, 0) (1, 0) (1, 1) (0, 1) (0, 0.5) (0.5, 0.5)
		} \closedcycle;
	\addplot[fill=blue, opacity=0.2, area legend] 
	coordinates {
        		(0, 0) (0.5, 0) (0.5, 0.5) (0, 0.5)
		} \closedcycle;
    \end{axis}
\end{tikzpicture}
\vspace{-4ex}
\caption{{\smaller Majority 0-1 loss}}
\vspace{-2ex}
\label{fig:M01_perfectly_attackble}
\end{subfigure}
\caption{Perfectly attackable conditions on $\alpha$ for each linear classification algorithm (colored in red). Assume $|\Dhp| = 75$, and $|\Dhn| = 25$.}
\label{fig:resilience_compare}
\vspace{-3ex}
\end{figure}
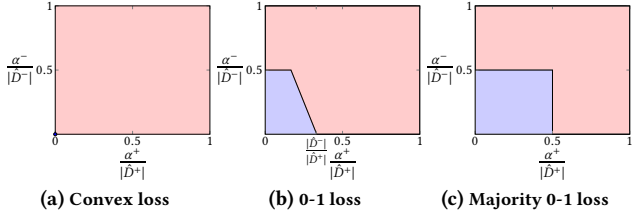

\vspace{-1.5ex}
\subsection{Convex Loss Linear Classification} \label{sec:feasibility_conv}
\vspace{-0.5ex}

In this section, the class of convex-loss linear classification algorithms is considered, where it is the collection of $P_{\Ls, \ell_c} \in \Ps_{\Ls, \Ss}$, where $\ell_{c}$ is any convex relaxation of a 0-1 loss function, such as a hinge loss function.
SVMs and a maximum likelihood learning of logistic regression belong to this class.
We prove that any algorithm in this class is perfectly attackable w.r.t. some attack condition where an attacker can tamper at least one feature vector.
Let $\Bs_{P_{\Ls, \ell_c}}$ be the attack condition for the convex-loss linear classification algorithms being perfectly attackable, and then the attack condition is formally stated as follows:
\begin{myproposition} \label{thm:feasibility_conv}
Let $\Bs_{P_{\Ls, \ell_c}}$ be the set of $\alpha$ that satisfies one of the following two conditions:
\eqa{
  \text{(\romannum{1})~} \alpha^+ > 0,~\text{(\romannum{2})~} \alpha^- > 0. \label{eq:A_conv}
}
Then, $P_{\Ls, \ell_c}$ is perfectly attackable w.r.t $\Bs_{P_{\Ls, \ell_c}}$ and resilient w.r.t. $\Bs_{P_{\Ls, \ell_c}}^{c}$.
\end{myproposition}
\vspace{-3ex}
\begin{proof}[proof sketch]
The idea of ``perfectly attackable'' proof is that 
for all $\N$ and $\alpha \in \Bs_{P_{\Ls, \ell_c}}$ we find some $\Dh$ and $\Dha$ where $V(P_{\Ls, \ell_{c}}(\Dha) \allowbreak | \N, \alpha) = 1$.
The ``resilient'' proof is trivial. 
See \Sec \ref{sec:feasibility_conv_proof} for details.
\end{proof}

\noindent 
This implies even though an attacker has weak ability to tamper training data, it can make the algorithm misclassify all positive or all negative feature vectors of un-attacked training data by tampering only one positive or negative feature vector (See \figref{fig:conv_perfectly_attackble} for the visualization of the perfectly attackable condition on $\alpha$).
For example, data-driven CPS that use SVMs to train intrusion detectors \cite{paridari2016cyber} can be vulnerable if an attacker can tamper at least one feature vector.
We note that convex-loss linear classification algorithms are not maximally resilient since $\Bs_{P_{\Ls, \ell_c}}^c \subset \bar{\Bs}^c$.


\vspace{-1.5ex}
\subsection{0-1 Loss Linear Classification} \label{sec:feasibility_01}
\vspace{-0.5ex}

A 0-1 loss linear classification algorithm is defined as $P_{\Ls, \ell_{01}} \allowbreak \in \Ps_{\Ls, \Ss}$, where
$\ell_{01}(\cdot)$ is a 0-1 loss function. 
We prove that 
the 0-1 loss linear classification algorithm is perfectly attackable w.r.t. some attack condition where
the number of tampered positive or negative feature vectors is greater than or equal to the half of $|\Dhp|$ or $|\Dhn|$, respectively, or 
the sum of the number of tampered positive feature vectors and the number of tampered negative feature vectors is greater than or equal to $|\Dhn|$ or $|\Dhp|$.
Let $\Bs_{P_{\Ls, \ell_{01}}}$ be the attack condition for the 0-1 loss linear classification algorithm being perfectly attackable, and then the attack condition is formally stated as follows:
\begin{myproposition} \label{thm:feasibility_01}
Given $| \Dhp |$ and $| \Dhn |$, 
let $\Bs_{P_{\Ls, \ell_{01}}}$ be the set of $\alpha$ that satisfies one of the following four conditions:
\eqa{
	\text{(\romannum{1})~}& \alpha^+ \geq \frac{1}{2} |\Dhp|,  & \text{(\romannum{2})~}& \alpha^- \geq \frac{1}{2} |\Dhn|, \nonumber \\
	\text{(\romannum{3})~}& \alpha^+ + \alpha^- \geq |\Dhn|, & \text{(\romannum{4})~}& \alpha^+ + \alpha^- \geq |\Dhp|. \label{eq:attackable_condition_01} 
}
Then, $P_{\Ls, \ell_{01}}$ is perfectly attackable w.r.t. $\Bs_{P_{\Ls, \ell_{01}}}$.
\end{myproposition}
\vspace{-3ex}
\begin{proof}[proof sketch]
For all $\N$ and $\alpha \in \Bs_{P_{\Ls, \ell_{01}}}$ we find some $\Dh$ and $\Dha$ where $V(P_{\Ls, \ell_{01}}(\Dha) | \N, \alpha) = 1$.
See \Sec \ref{sec:feasibility_01_proof} for details.
\end{proof}

\noindent
This proposition implies the \zo loss linear classification is strictly more resilient than convex one (See \figref{fig:01_perfectly_attackble} for comparison). Thus, different to the convex case, tampering single feature vector is not critical for the \zo loss linear classification. 
This means any CPS using convex linear classification algorithms  \cite{paridari2016cyber,chen2015deepdriving,seo2014predicting} can be converted into the \zo linear classification algorithm to defend against the single feature vector tampering; 
however, neither approach can provide maximal resilience due to $\Bs_{P_{\Ls, \ell_{01}}}^c \subset \bar{\Bs}^c$.


\vspace{-1ex}
\section{Resilient Linear Classification} \label{sec:proposed}
In this section, we propose a maximally resilient linear classification algorithm.
A majority $0$-$1$ loss linear classification is defined as $P_{\Ms, \ell_{01}} \in \Ps_{\Ls, \Ss}$, where $\Ms$ denotes a majority constraint that restricts a feasible set of classifiers by only allowing a classifier that correctly classifies at least half of positive and negative feature vectors, according to
\eqa{
\!\!\!\Ms\!=\!\left\{ h\!\in\!\Ls \middle| \Rh_{\ell_{01}}(h|\Dhap)\!<\!\frac{1}{2} \text{~and~} \Rh_{\ell_{01}}(h|\Dhan)\!<\!\frac{1}{2} \right\}\!. \!\!
}
In the following subsections, the resilience proof and the worst-case resilience bound of the majority 0-1 classification are provided.

\vspace{-1.5ex}
\subsection{Resilience of Majority 0-1 Loss Linear Classification}
\vspace{-0.5ex}

The majority 0-1 loss linear classification is perfectly attackable w.r.t. some attack condition where an attacker can manipulate greater than or equal to the half of $|\Dhp|$ or $|\Dhn|$. 
Let $\Bs_{P_{\Ms, \ell_{01}}}$ be the attack condition for the majority 0-1-loss linear classification algorithms being perfectly attackable, and then the attack condition is formally stated as follows:
\begin{mytheorem} \label{thm:feasibility_maj}
Given $| \Dhp |$ and $| \Dhn |$, 
let $\Bs_{P_{\Ms, \ell_{01}}}$ be
\eqa{
\left\{ \alpha \middle| \alpha^+ \geq \frac{1}{2} |\Dhp| \text{~or~} \alpha^- \geq \frac{1}{2} |\Dhn| \right\}. 
}
Then, $P_{\Ms, \lzo}$ is perfectly attackable w.r.t. $\Bs_{P_{\Ms, \ell_{01}}}$ and resilient w.r.t. $\Bs_{P_{\Ms, \ell_{01}}}^{c}$.
\end{mytheorem}
\vspace{-3ex}
\begin{proof}[proof sketch]
The ideal of ``perfectly attackable'' proof is that
for all $\N$ and $\alpha \in \Bs_{P_{\Ms, \ell_{01}}}$ we find some $\Dh$ and $\Dha$ where $V(P_{\Ms, \ell_{01}}\allowbreak(\Dha)  | \N, \alpha) = 1$.
For the ``resilient'' proof, we exploit the property of the majority constraint.
See \Sec \ref{sec:feasibility_maj_proof} for details.
\end{proof}
\noindent
This result shows that the majority \zo loss linear classification algorithm is more resilient than traditional linear classification algorithms, which is also illustrated in \figref{fig:M01_perfectly_attackble}.
Furthermore, it achieves the maximal resilience condition (\Thm \ref{thm:maximal_resilience}) due to $\Bs_{P_{\Ms, \ell_{01}}} = \bar{\Bs}$, showing this algorithm achieves the maximal resilience attack condition.

\vspace{-1.5ex}
\subsection{Robustness of Resilient Classification} \label{sec:robustness_proof_maj}
\vspace{-0.5ex}


If a classification algorithm is resilient, it is worth analyzing the degree of resilience.
If $\alpha \in \As_{P_{\Ms, \ell_{01}}}$, where $\As_{P_{\Ms, \ell_{01}}} = \Bs_{P_{\Ms, \ell_{01}}}^{c}$, then the worst-case resilience bound $g$ of the majority \zo loss classification algorithm is nearly proportional to the tampering ability of an attacker, which is formally stated as follows:
%
\begin{mytheorem} \label{thm:robustness_maj}
Given $|\Dhp|$, $|\Dhn|$, and $\alpha \in \As_{P_{\Ms, \ell_{01}}}$, the resilience bound $g$ of $P_{\Ms, \ell_{01}}$ can be computed as follows:
\vspace{-1ex}
	\eqa{
		\!\!\!\!g(\N, \alpha) \! = \! \max \!
			\Bigg(
  			&\frac{\min\left( 2\alpha^+ + \alpha^-, \alpha^+ + \frac{|\Dhp| - 1}{2}\right)}{|\Dhp|},  \nonumber
			 \\ 
			&\frac{\min\left( \alpha^+ + 2\alpha^-, \alpha^- + \frac{|\Dhn| - 1}{2}\right)}{|\Dhn|}
			\Bigg). \label{eq:robustness_bound}
	}
\end{mytheorem}
\vspace{-3ex}
\begin{proof}[proof sketch]
To prove $V(\cdot)$ is bounded by $g(\cdot)$ for all $\N$ and $\alpha \in \As_{P_{\Ms, \ell_{01}}}$, we exploit the optimality condition of an optimal classifier $P_{\Ms, \ell_{01}}\allowbreak(\Dha)$ and the property of the majority constraint. 
To prove that the bound is tight for all $\N$ and $\alpha \in \As_{P_{\Ms, \ell_{01}}}$, we find some $\Dh$ and $\Dha$ where $V(\cdot) = g(\cdot)$.
See \Sec \ref{sec:robustness_maj_proof} for details.
\end{proof}

\noindent
This theorem shows that if $\alpha \in \As_{P_{\Ms, \ell_{01}}}$, the resilience bound is non-trivial.
Also, it shows that even if the attacker capability parameter $\alpha$ is restricted (\ie $\alpha \in \As_{P_{\Ms, \ell_{01}}}$) to ensure that the algorithm is resilient w.r.t. $\As_{P_{\Ms, \ell_{01}}}$, the tampered portion of training data still affects on the accuracy of the algorithm. 
Finally, we note that the resilience bound $g$ is tight. 
 

\vspace{-1ex}
\section{Case Study} \label{sec:exp}


In this section, we validate the proven resilience of algorithms experimentally.
Qualitative results on synthetic data are presented in \figref{fig:exp_qualitative} and results on a real-world retrospective arrhythmia data are shown in \tabref{tab:exp_qualitative}.

The majority \zo loss linear classification algorithm is formulated in the following mixed integer linear program (MILP).
\begin{align*}
	\min_{\h, \e, \z}~&~ \bm{1}^\top \z + \lambda \| \h \|_2 \\
	\text{s.t.}
	~&~ \forall i, e_i \geq 1 - y_i \h^\top \x_i, 
	~     -\delta \z \leq \e \leq \delta \z \\
	~&~ \bm{1}_+^{\top} \z \leq \frac{1}{2} (|\Dhap| - 1),
	~\bm{1}_-^{\top} \z \leq \frac{1}{2} (|\Dhan| - 1),
\end{align*}
where 
$(\x_i, y_i)$ is an $i$th training data pair, 
$\h \in \realnum^p$ is a real-valued classifier, 
$\e \in \realnum^{|\Dha|}$ denotes a scaled classification error,
$\z \in \{0, 1\}^{|\Dha|}$ is a vector that indicates misclassification of each training data pair,
$\lambda$ is a regularization constant, set to zero, and
$\delta$ is a sufficiently large positive constant, where $\delta = 10^3$.
$\bm{1}_+$ and $\bm{1}_-$ represent vectors where a $j$th element is filled with one if $y_j = +1$ and $y_j = -1$, respectively, and zeros elsewhere.
We note that the \zo loss linear classification algorithm is formulated in the same way to the above MILP except for the last two constraints (See \Sec \ref{sec:milp_proof}), related to the majority constraint, and we adopt a standard SVMs formulation \cite{cortes1995support} without a regularization term for fair comparison.
Theoretically, the performance of the \zo loss linear classification algorithm is as good as that of the convex loss linear classification algorithms \cite{bartlett2006convexity}. If there are no attack and no error, the \zo loss linear classification algorithm is same as the majority \zo loss linear classification algorithm since the last two constraints of MILP are not activated if there are no attacks and no error. 

In experiments, we consider two types of attacks: a \emph{point attack} and an \emph{overlap attack}, which are concrete instances of a \tma. 
The point attack is an attack that manipulates a single feature vector to be located far from the training data as illustrated in \figref{fig:exp_qualitative}. 
The attacked single feature vector is chosen and tampered as follows. 
Let $\alpha^+ = 1$, and $\bar{\x}^+$ and $\bar{\x}^-$ be the mean of positive and negative feature vectors, respectively.
Any positive feature vector is chosen and replaced to a scaled vector $\sigma \x$ where
the scaled vector is on the half-line from $\bar{\x}^+$ to the direction of $\bar{\x}^- - \bar{\x}^+$, and the scale value $\sigma$ is a sufficiently large scalar.

The overlap attack is an attack that manipulates positive and/or negative feature vectors to be overlapped negative and/or positive feature vectors, respectively, as illustrated in \figref{fig:exp_qualitative}.
The overlap attack is briefly described as follows: 
when $\alpha = (\alpha^+, \alpha^-)$, $\alpha^+$ and $\alpha^-$ number of positive and negative feature vectors are randomly chosen for tampering, respectively. The chosen positive and negative feature vectors are randomly overlapped to negative and positive un-attacked feature vectors, respectively. These steps are repeated until a target classification algorithm achieves a maximum desired resilience value $V(\cdot)$. 

\begin{figure}[tb!]
\centering
\includegraphics[width=0.75\linewidth]{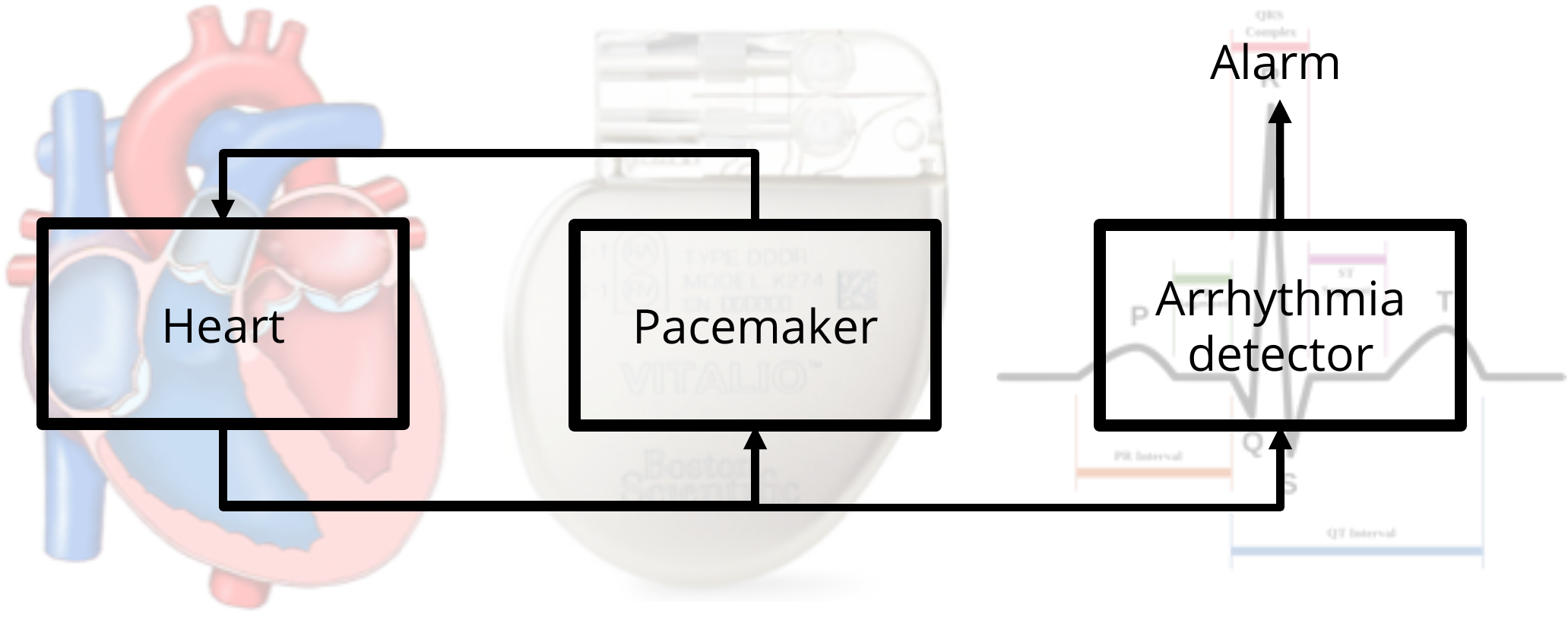}
\vspace{-1.5ex}
\caption{Pacemaker with an Arrhythmia detector. }
\vspace{-4ex}
\label{fig:case_study}
\end{figure}

\begin{table}[tb!]
\centering
\begin{tabular}{M{0.15\textwidth}||M{0.08\textwidth}M{0.08\textwidth}M{0.08\textwidth}}
	& \multicolumn{3}{c}{Approach} \\ \cline{2-4}
	\multirow{-2}{*}{Attack Type} & 
			SVMs & 
			$0$-$1$ & 
			$0$-$1$ with majority \\ \toprule
	No Attack \footnotesize{$(\alpha^+,\alpha^-) = (0, 0)$} & 
			${0.0}$ & 
			${0.0}$ & 
			${0.0}$    
			\\ \hline
	Point Attack {\footnotesize$(\alpha^+,\alpha^-) = (0, 1)$}&
			{\color{red}$\mathbf{1.0}$} & 
			${0.0270}$ & 
			${0.0270}$    
			\\ \hline
	Overlap Attack {\footnotesize$(\alpha^+,\alpha^-) = (16, 21)$}&
			{\color{red}$\mathbf{1.0}$} & 
			{\color{red}$\mathbf{1.0}$} & 
			${0.5946}$    
			\\ \bottomrule
\end{tabular}
\caption{The resilience metric $V(\cdot)$ of each linear classification algorithm for Arrhythmia detection {\protect\cite{guvenir1997supervised}} under the specified attacks.}
\label{tab:exp_qualitative}
\vspace{-5ex}
\end{table}

\begin{figure*}[t!]
\centering
\begin{tabular}{M{0.1\textwidth}|M{0.2\textwidth}||M{0.2\textwidth}M{0.2\textwidth}M{0.2\textwidth}}
	\toprule
	 & & \multicolumn{3}{c}{Approach} \\ \cline{3-5}
	\multirow{-2}{*}{Attack Type} & \multirow{-2}{*}{Training Data} & 
			SVMs & 
			0-1 & 
			0-1 with majority\\ \toprule
	No Attack {\small$(\alpha^+,\alpha^-) = (0, 0)$} & 
			\includegraphics[height=2.3cm,keepaspectratio]{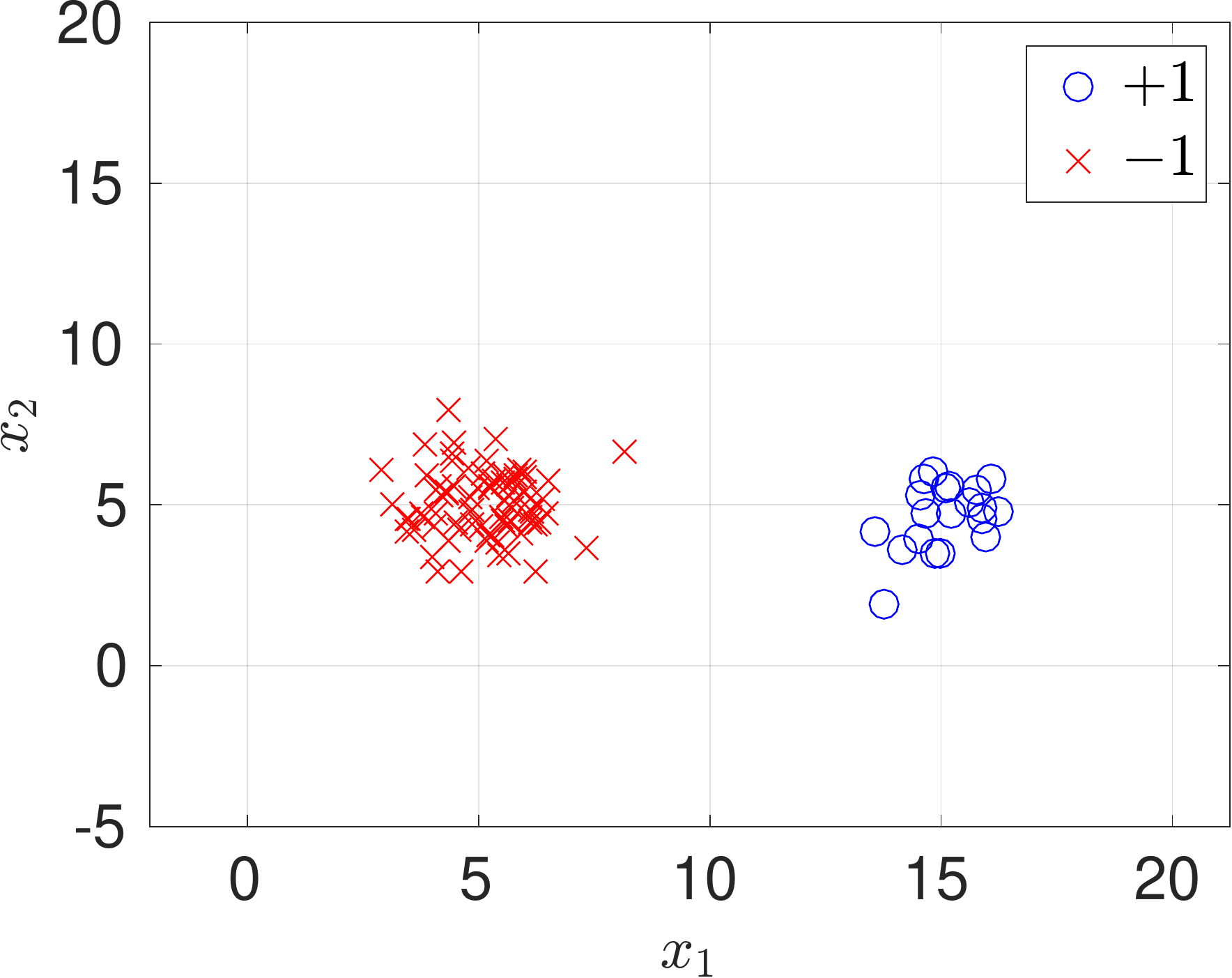} & 
			\includegraphics[height=2.3cm,keepaspectratio]{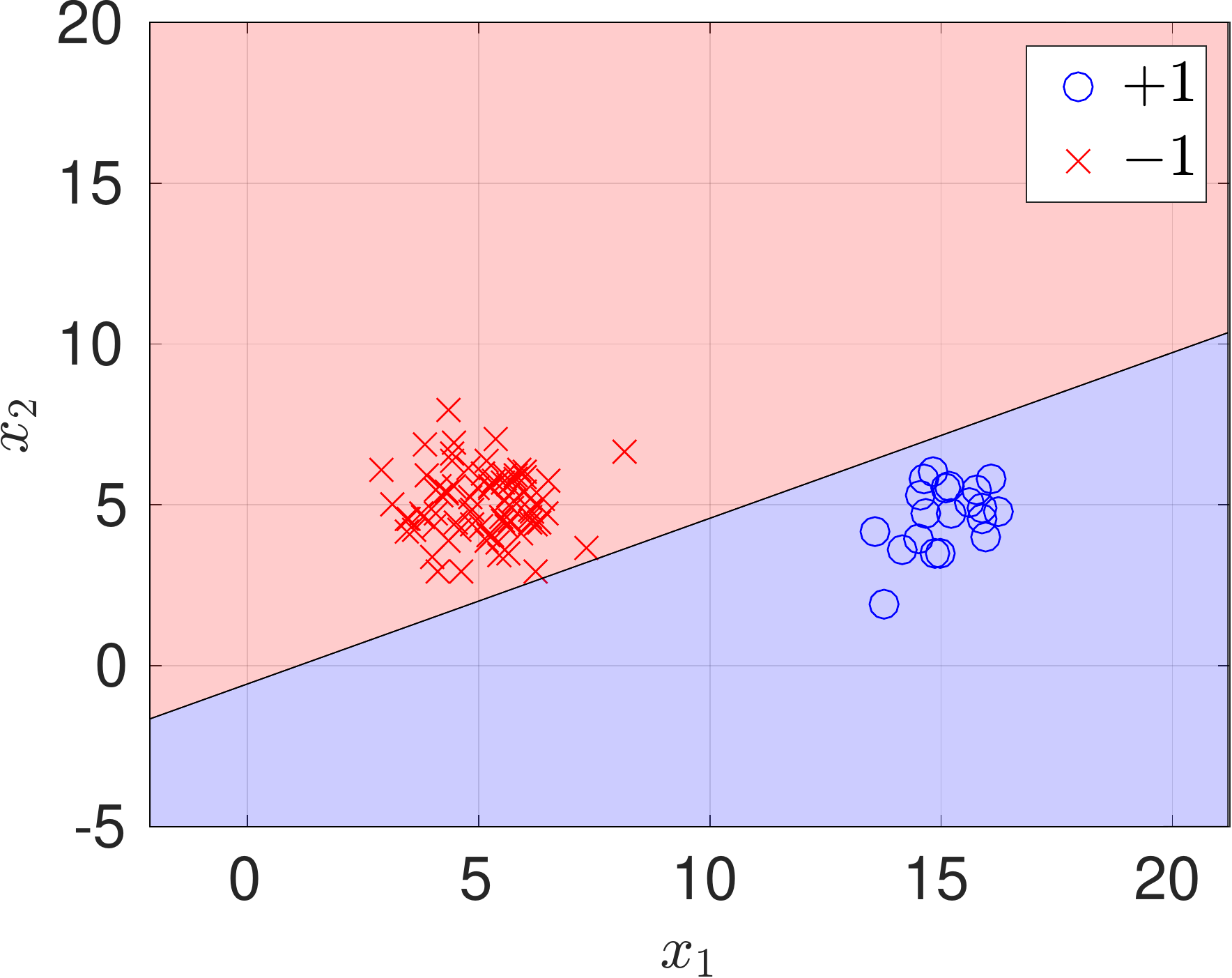} & 
			\includegraphics[height=2.3cm,keepaspectratio]{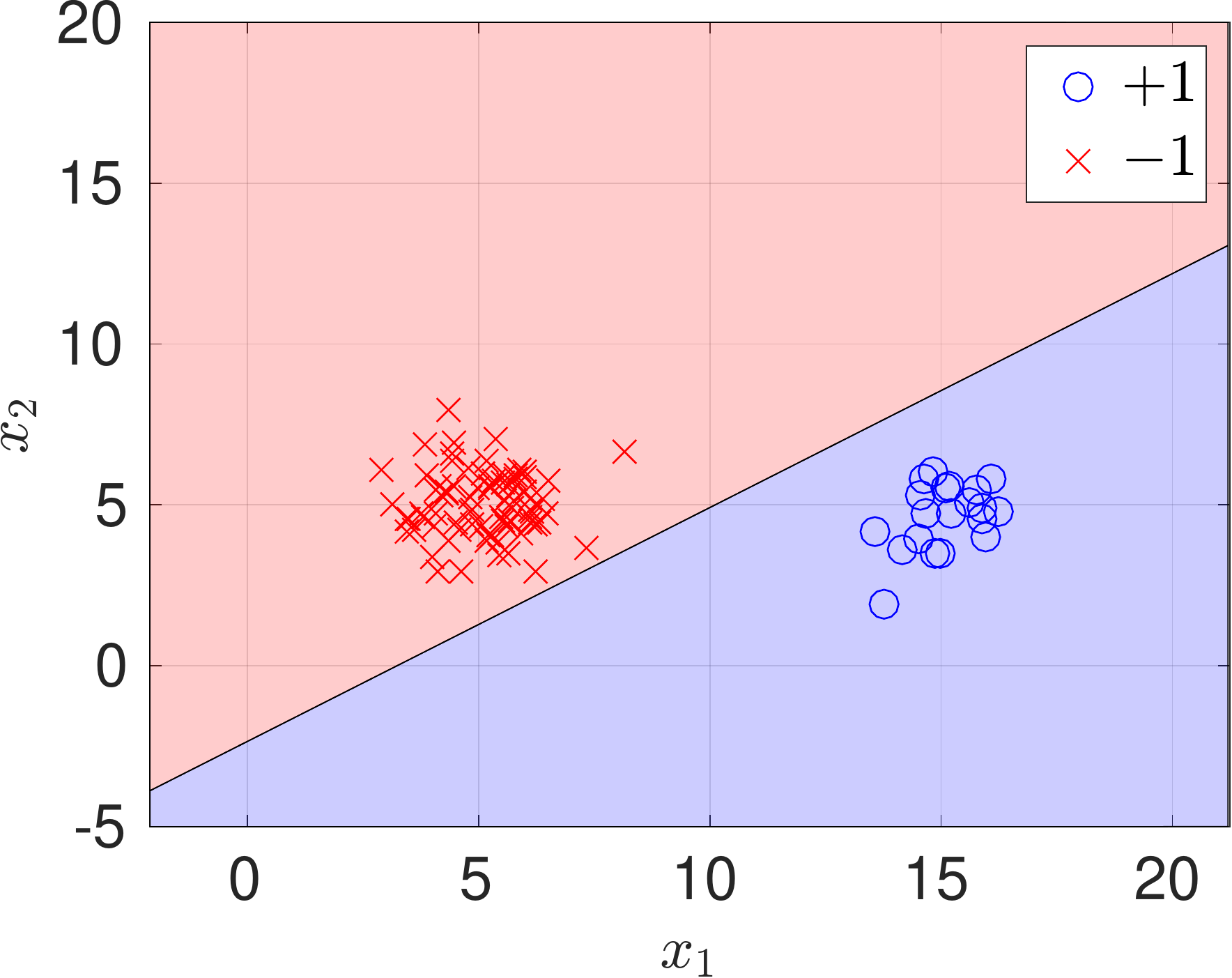} & 
			\includegraphics[height=2.3cm,keepaspectratio]{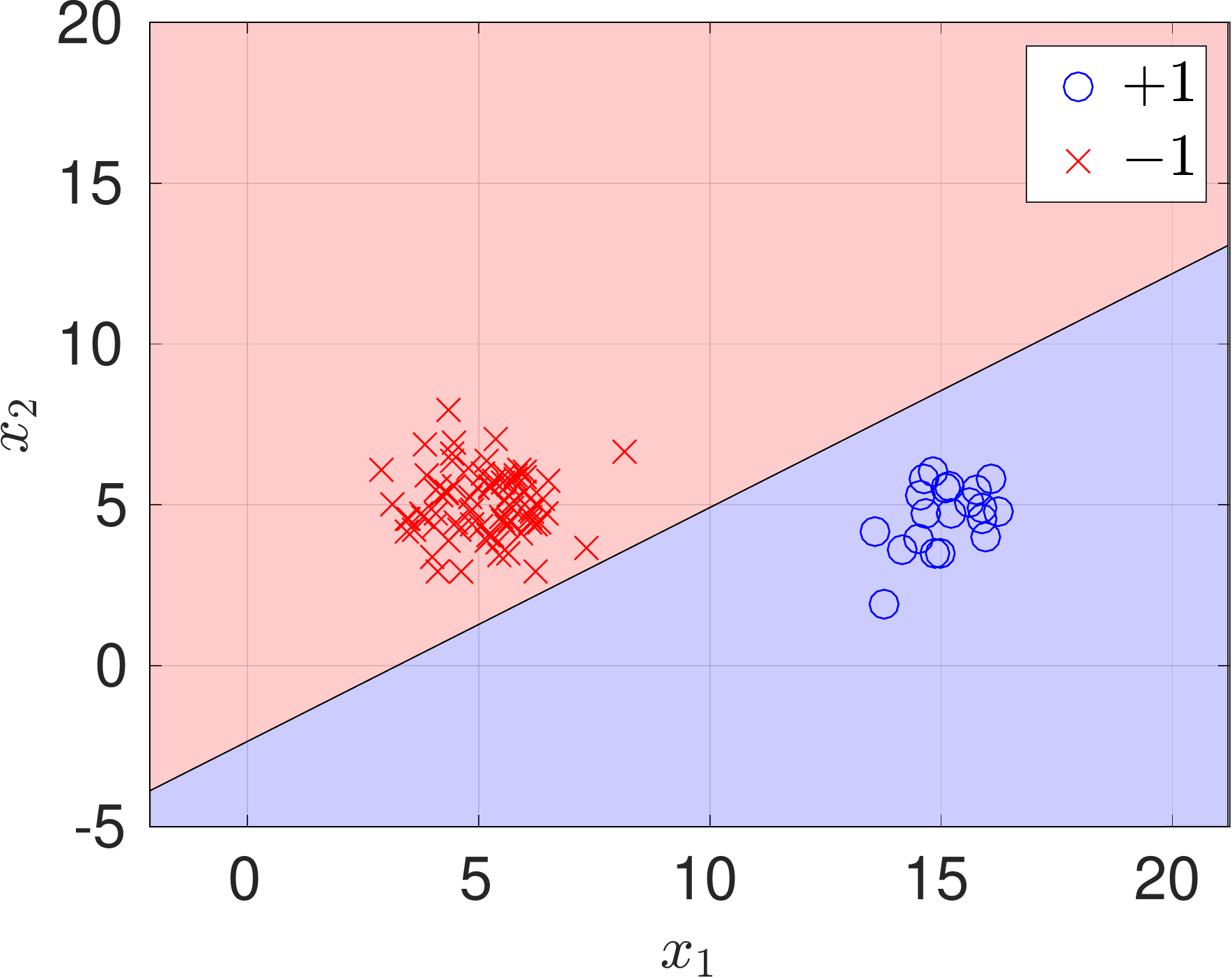} \\
	Point Attack {\small$(\alpha^+,\alpha^-) = (0, 1)$} &
			\includegraphics[height=2.3cm,keepaspectratio]{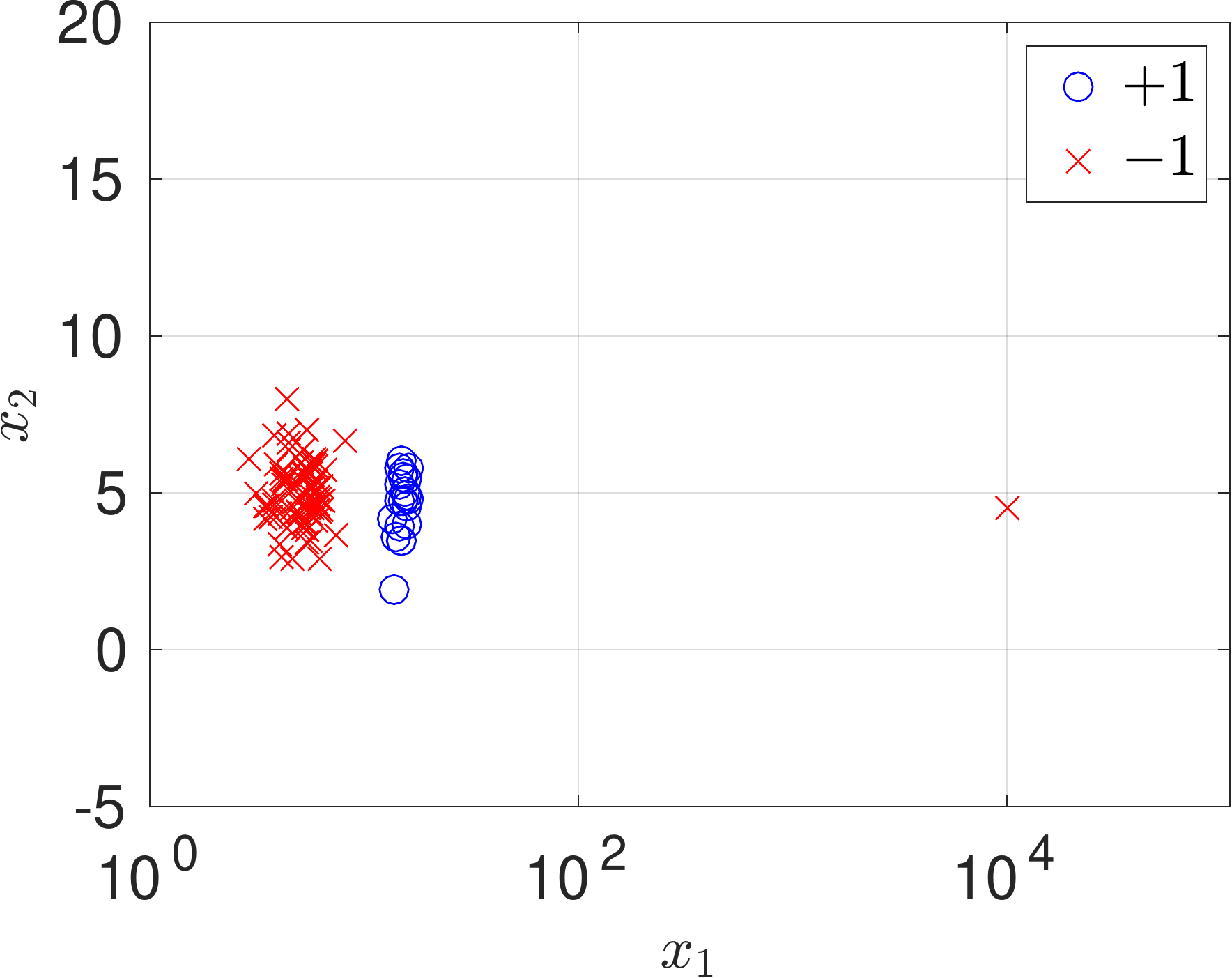} & 
			\includegraphics[height=2.3cm,keepaspectratio]{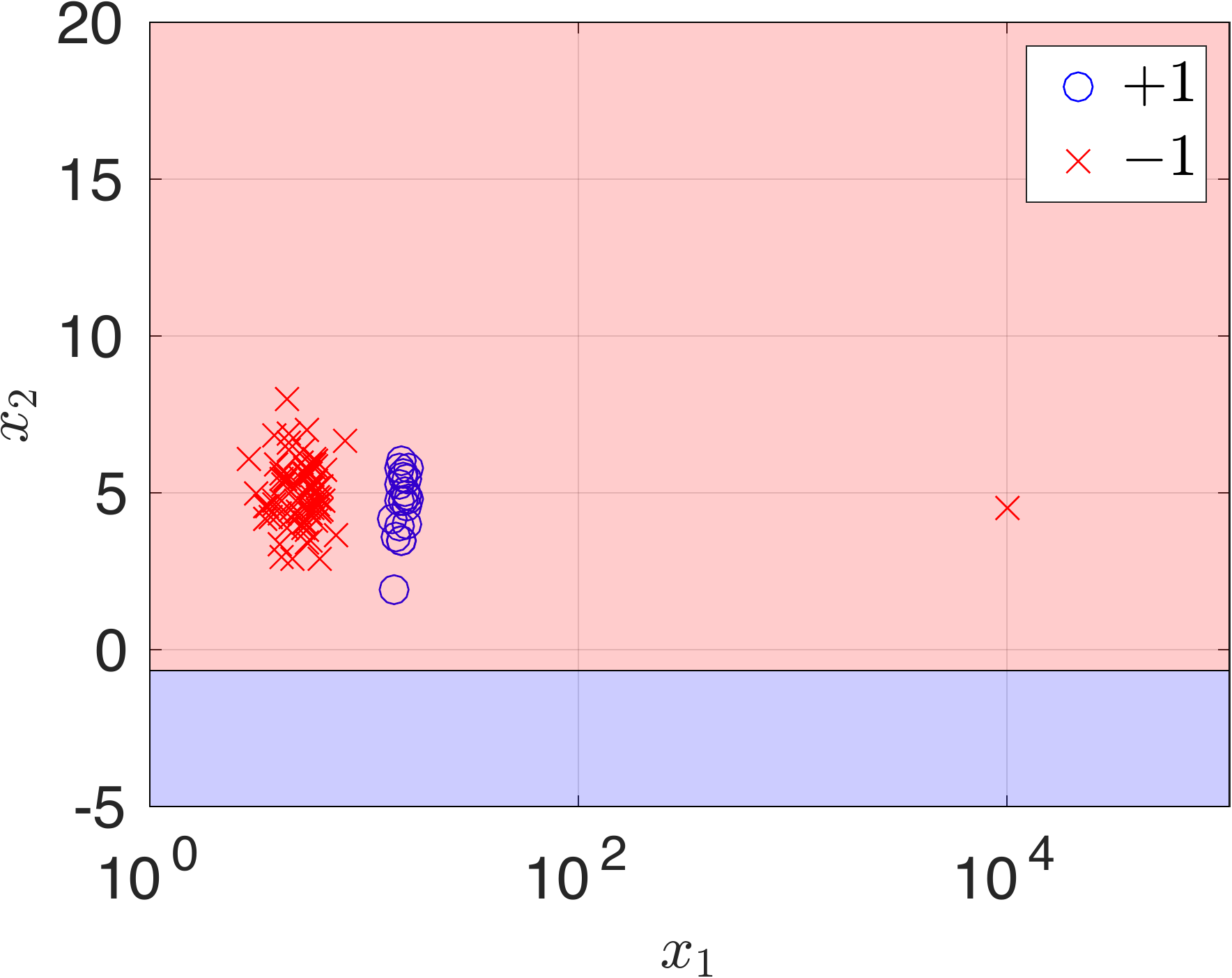} & 
			\includegraphics[height=2.3cm,keepaspectratio]{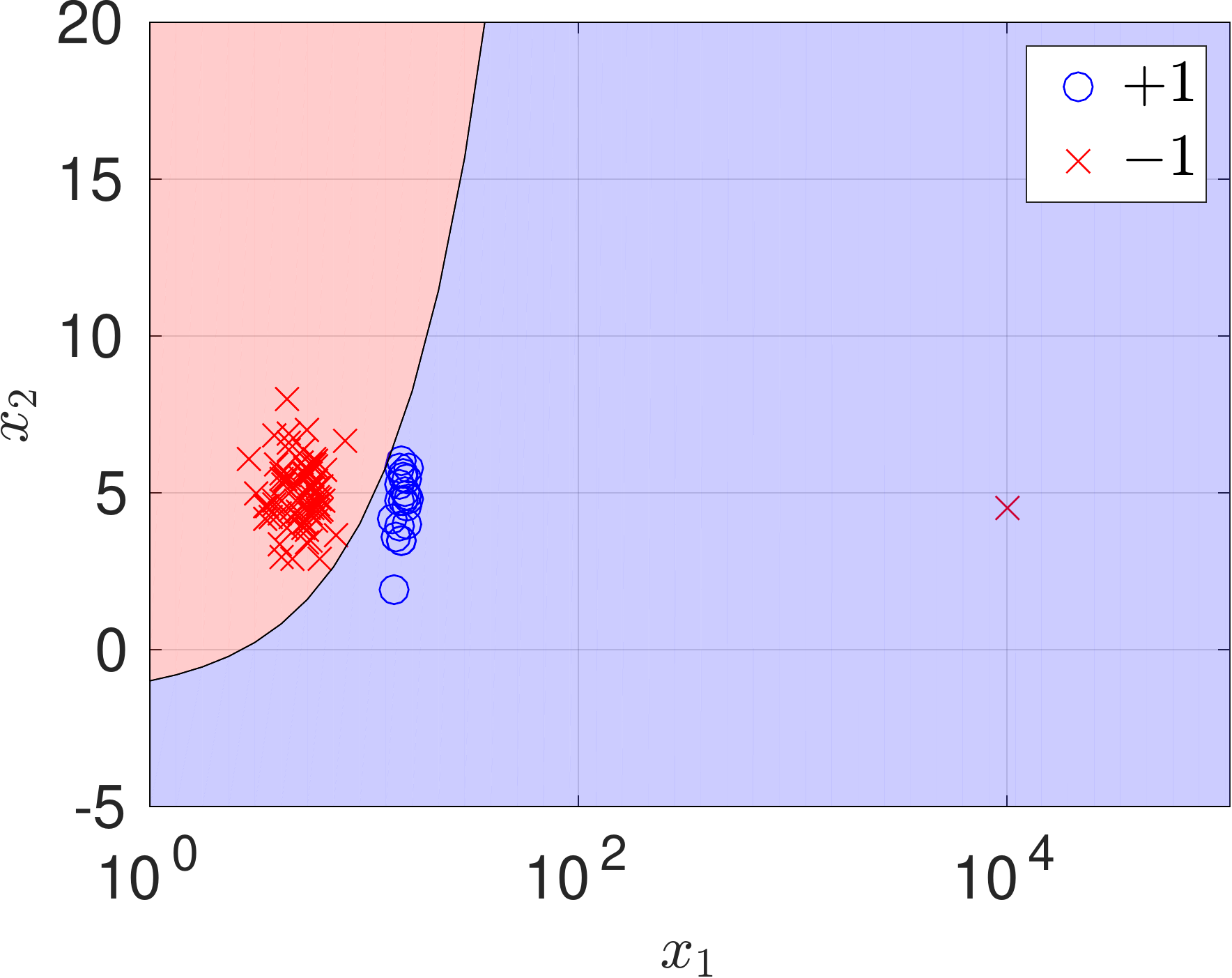} & 
			\includegraphics[height=2.3cm,keepaspectratio]{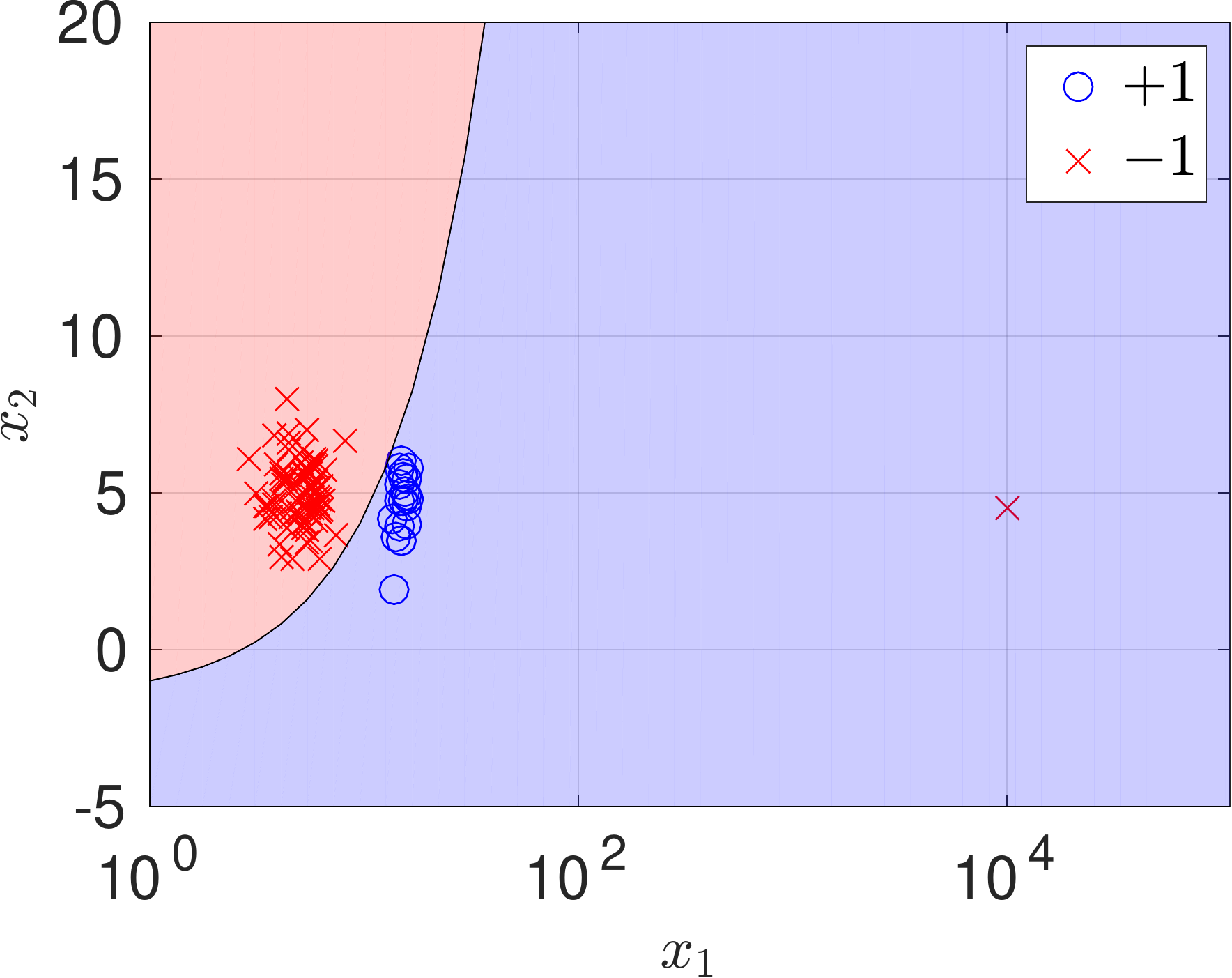} \\ 
	Overlap Attack  {\small$(\alpha^+,\alpha^-) = (0, 24)$}&
			\includegraphics[height=2.3cm,keepaspectratio]{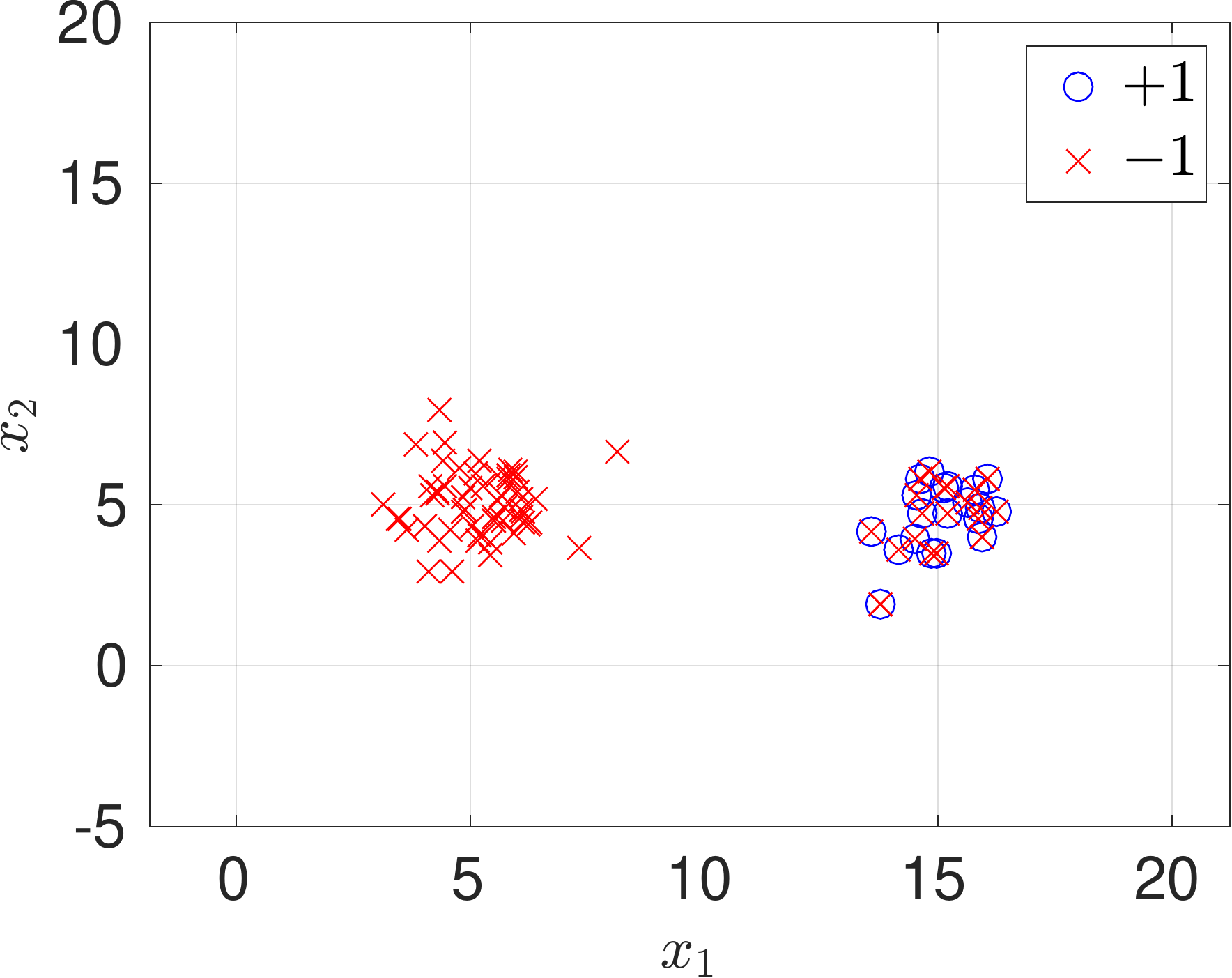} & 
			\includegraphics[height=2.3cm,keepaspectratio]{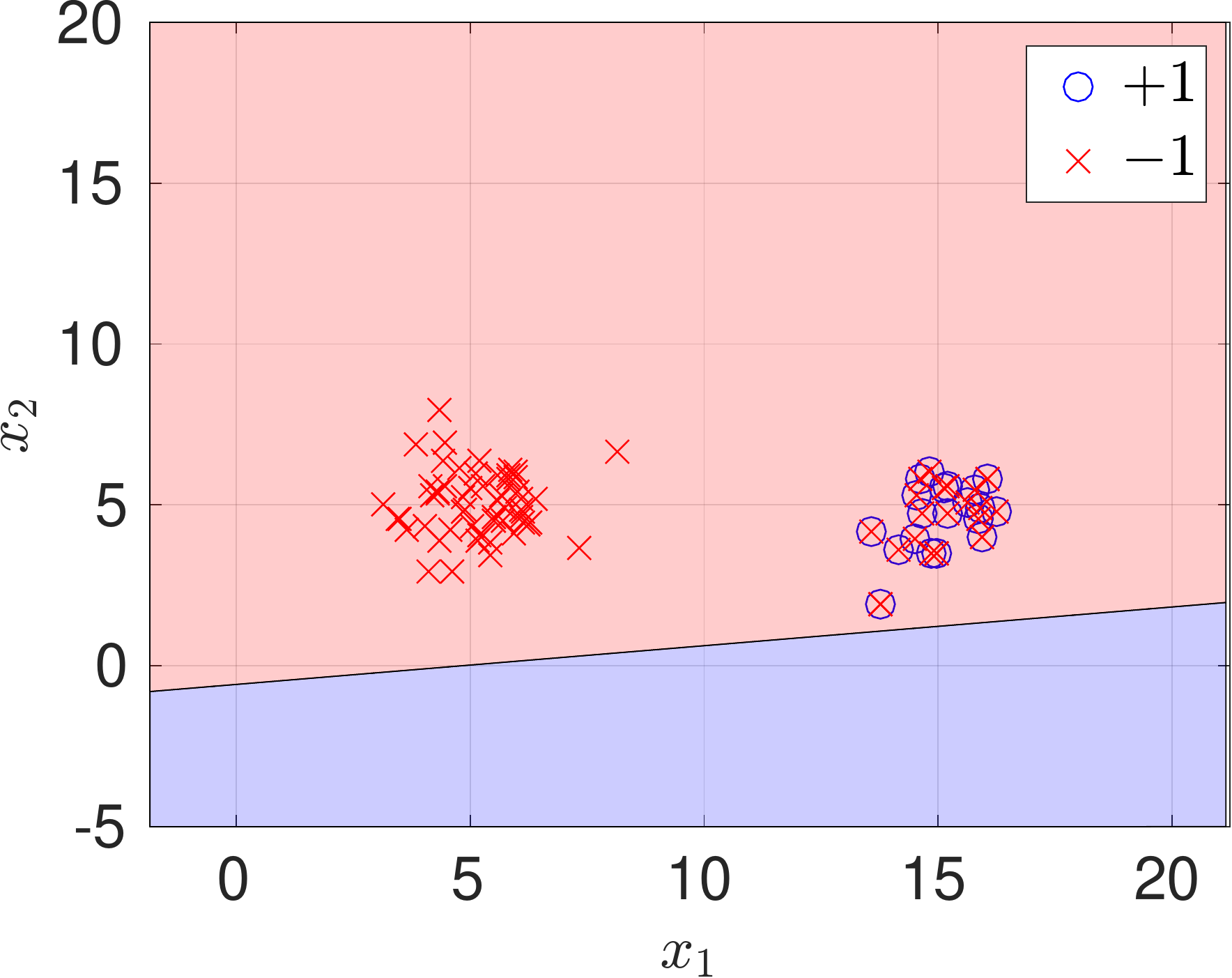} & 
			\includegraphics[height=2.3cm,keepaspectratio]{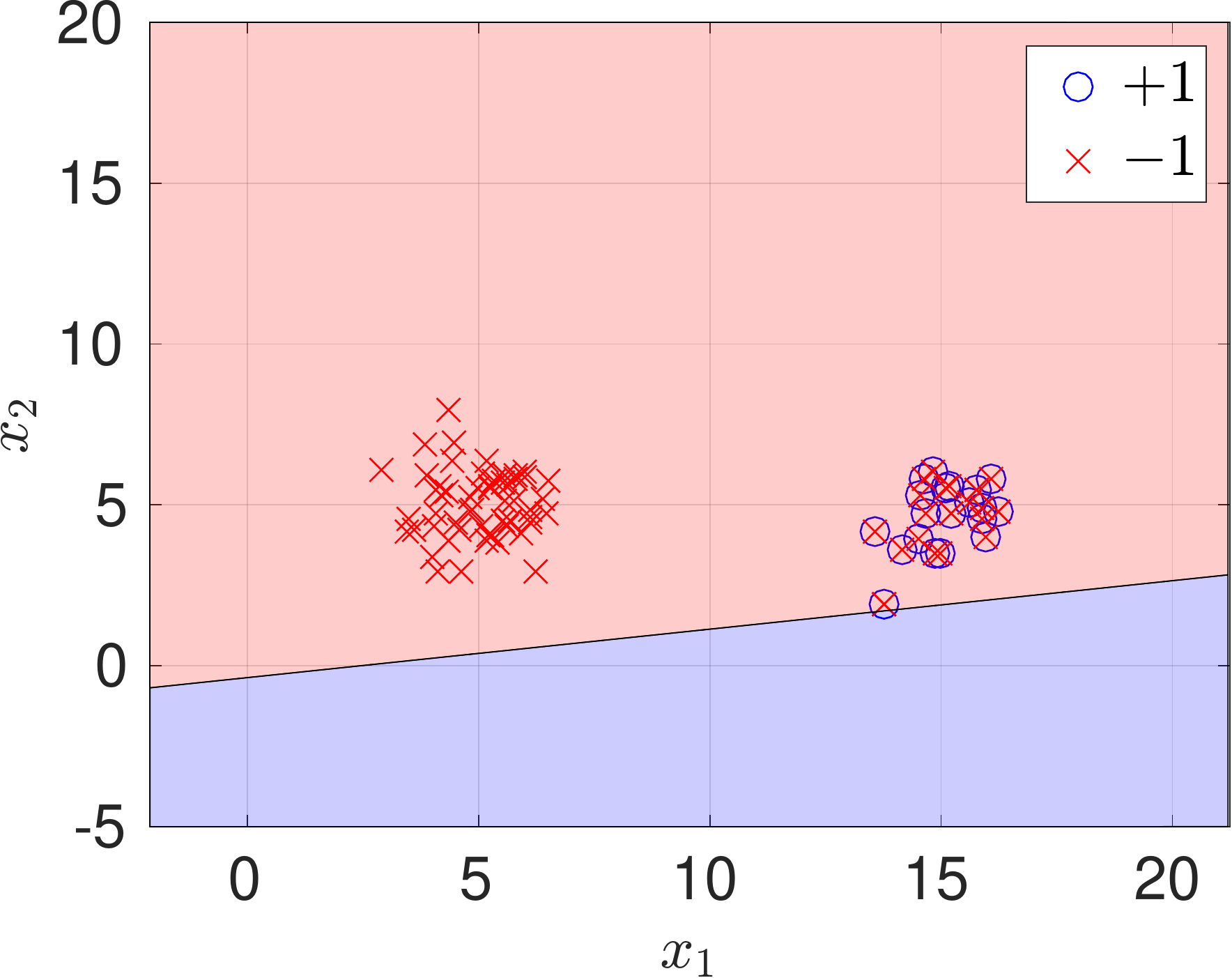} & 
			\includegraphics[height=2.3cm,keepaspectratio]{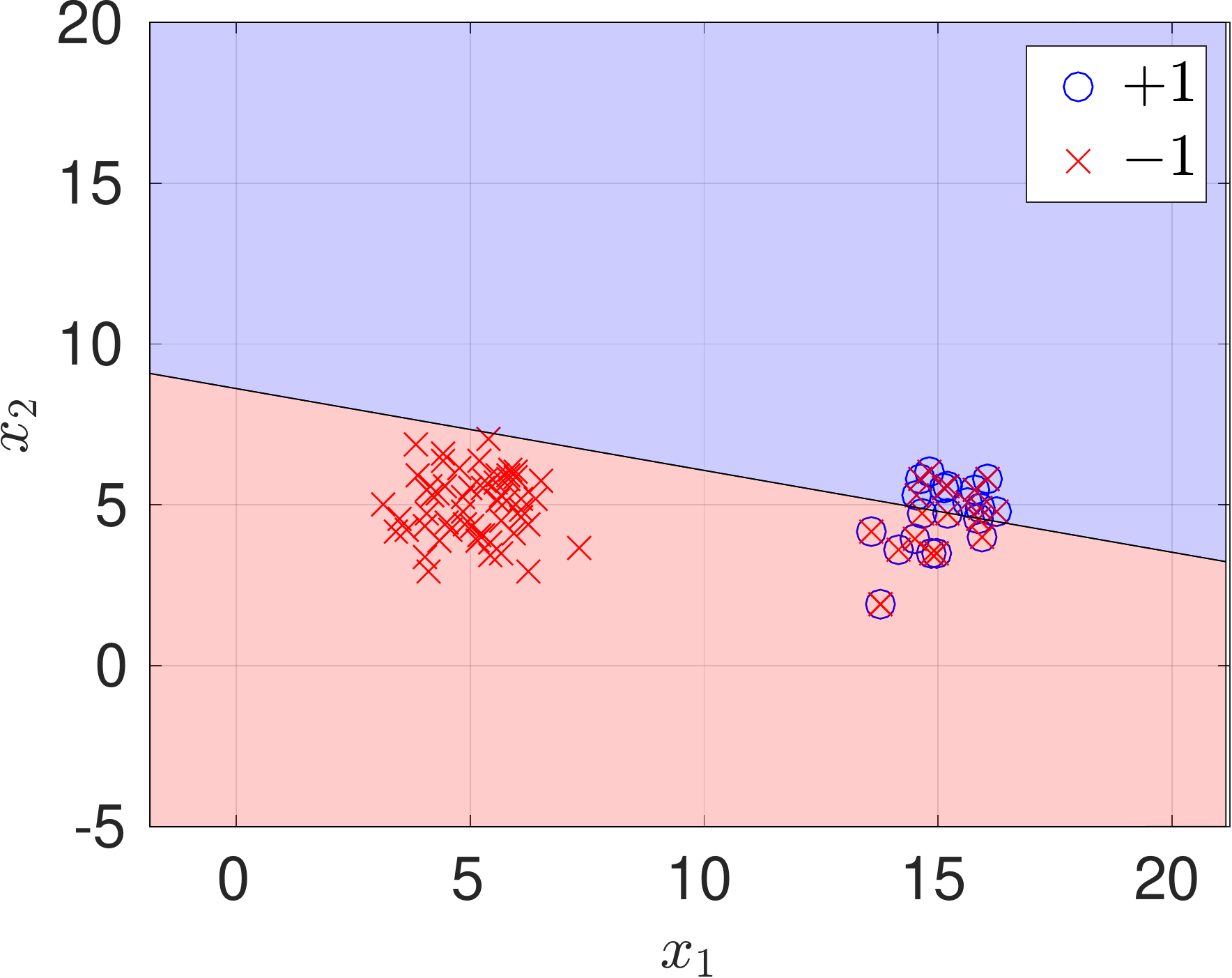} \\ 
	\bottomrule
\end{tabular}
\vspace{-3ex}
\caption{The resilience of each linear classification algorithm under the specified attacks. The blue or red mark represents a positive or negative feature vector, respectively. The feature vector in the blue or red region is classified as positive or negative, respectively.}
\label{fig:exp_qualitative}
\vspace{-1.5ex}
\end{figure*}

\begin{figure}[tb!]
\centering
\includegraphics[width=0.6\linewidth]{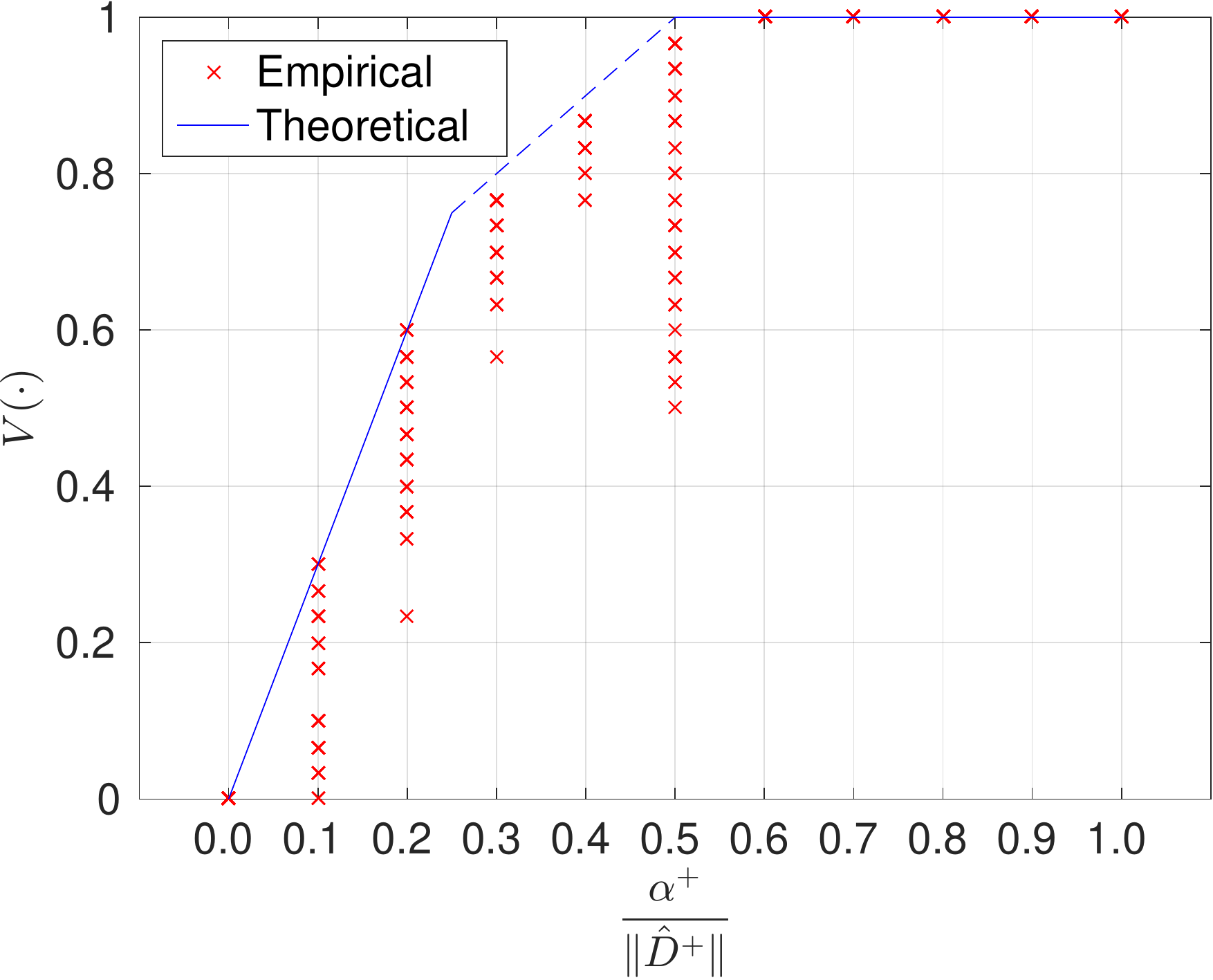}
\vspace{-2ex}
\caption{The degree of the resilience of $P_{\Ms, \lzo}$ in the resilience metric $V(\cdot)$ with respect to the ability of an attacker. The blue solid and dashed line represents the theoretical robustness bound (\eqnref{eq:robustness_bound}) and the red cross means the empirically evaluated feasible resilience. Assume $|\Dhp| = |\Dhn|$ and $\alpha^+ = \alpha^-$.}
\vspace{-3ex}
\label{fig:robust}
\end{figure}

\textbf{Synthetic data.}~
In \figref{fig:exp_qualitative}, the classification results of each linear classification algorithm, such as SVMs, the \zo loss linear classification, and the majority \zo loss linear classification, are illustrated with different types of an attack.
The original training data without attacks is randomly drawn from two Gaussian distributions, as illustrated in the first column and the first row, where $|\Dhp| = 20$ and $|\Dhn| = 80$. 
When there is no attack (the first row in \figref{fig:exp_qualitative}), all three algorithms correctly classify training data. 
If there is a point attack (the second row in \figref{fig:exp_qualitative}), only SVMs algorithm is affected by the attack, outputting a classifier that misclassifies all positive feature vectors of un-attacked training data. 
When an overlap attack (the third row in \figref{fig:exp_qualitative}) is applied, where $\alpha^- = 24$, both SVMs and the \zo loss linear classification output classifiers that misclassifies all positive feature vectors of un-attacked training data while the majority \zo loss classification algorithm still correctly classifies the portion of the positive feature vectors of un-attacked training data, showing that the majority \zo loss classification algorithm is more resilient than others.

Moreover, using the synthetic data, the theoretical worst-case resilience bound (\eqnref{eq:robustness_bound}) of the majority \zo loss linear classification is experimentally shown in \figref{fig:robust}. The blue line represents the theoretical worst-case resilience bound. Red points are the resilience $V(\cdot)$ over the corresponding $\alpha$. Specifically, 
100 different $\Dh$ are randomly generated, where $\Dhp$ and $\Dhn$ are drawn from two Gaussian distributions of positive and negative labels, respectively. 
For each $\Dh$ and for each $\alpha^+$, which ranges from 0 to the total number of positive feature vectors, 
an attacker moves $\alpha^+$ number of positive feature vectors beyond the negative features in 100 different ways to obtain $\Dha$ so that positive and negative feature vectors cannot be linearly separable. By taking the maximum of $V(\cdot)$ for 100 different $\Dh$ and 100 different $\Dha$, the resilience $V(\cdot)$ is obtained for each $\alpha^+$, which is represented in a red cross. In \figref{fig:robust}, the red crosses do not excess the theoretical bound and the increasing trend follows the bound.



\textbf{Medical data.}~
We evaluated the resilience of traditional linear classification algorithms and the proposed algorithm using arrhythmia dataset.
The arrhythmia, a.k.a irregular heartbeat, is a condition of the heart in which the heartbeat is irregular.
An arrhythmia detector cooperated with logs from pacemaker can reduce stroke and death rate \cite{glotzer2003atrial}.
To design such a detector, electrocardiogram (ECG) training data can be collected from logs of the pacemaker (\figref{fig:case_study}) whether ECG data is normal or abnormal (\eg atrial fibrillation or sinus tachycardia). But, if the pacemaker is vulnerable, the training data can be tampered to hinder to detect arrhythmia, possibly leading to death.

In \tabref{tab:exp_qualitative}, we have compared the resilience $V(\cdot)$ of each algorithm on real medical dataset. 
Arrhythmia dataset \cite{guvenir1997supervised}, which can be found at the UCI machine learning repository \cite{Lichman2013}, is used for evaluating the resilience of each algorithm. 
The Arrhythmia dataset is preprocessed as follows. Due to the computational limitation to solve the MILP, we use 20 percent of training data (\ie $|\Dhp| = 37$ and $|\Dhn| = 49$) and select features from 40th and 99th for training classifiers. 
We've obtained the same results as illustrated with synthetic data.
SVMs algorithm outputs a classifier that misclassifies all positive or negative feature vectors under both a point attack and an overlap attack.
\zo loss linear classification does not affect on a point attack but outputs a classifier that misclassifies all positive or negative feature vectors when an overlap attack is applied by tampering 43.2 and 42.8 percent of positive and negative feature vector, respectively.
However, the majority \zo loss linear classification still correctly classifies the portion of positive and negative feature vectors even though about 43 percent of training data were tampered. 
We emphasize that $V(\cdot)$ values in \tabref{tab:exp_qualitative} are not prediction results, but they are evaluated over training data, making $V(\cdot) = 0$ possible. However, a higher $V(\cdot)$ value implies higher prediction error. 

\textbf{Comparison with \cite{kearns1993learning}.}
Kearns and Li's paper \cite{kearns1993learning} analyzes a binary classification problem under the malicious error (ME) model, but our paper analyzes a binary linear classification problem under a \tma, which is a general case of the ME model. Here, we compare each paper's result by providing an example.
Assume a binary linear classification problem under the ME model, where $|\Dhp| = |\Dhn| = 50$ and $\alpha^{+} = \alpha^{-} = 10$.
Kearns and Li's paper states that if a designer wants to have the expected accuracy of $0.9$, then $\frac{\alpha^{+}}{|\Dhp|} = \frac{\alpha^{-}}{|\Dhn|} < \frac{0.1}{1 + 0.1} \approx 0.091$ regardless of a classification algorithm. This means at most $0.091$ percent of training data can be tampered to guarantee the expected accuracy. However, this does not state anything on the expected accuracy when $\frac{\alpha^{+}}{|\Dhp|} \geq 0.091$.
In comparison to this, our paper implies that, in the case of the majority \zo linear classification algorithm, $g(|\Dh|, (\alpha^{+}, \alpha^{-})) = 0.6$. This means that a designer can expect the accuracy on training data that is at least $1 - 0.6 = 0.4$. This further implies the expected minimum accuracy can be approximately $0.4$ when $\frac{\alpha^{+}}{|\Dhp|} = 0.2 \geq 0.091$. We note that the connection between the accuracy on training data (\ie the performance measure of this paper) and the expected accuracy (\ie the performance measure of traditional classification) can be found in \Sec \ref{sec:connection_stl}.



\vspace{-1ex}
\section{Conclusions}


In particularly, the incorrect decisions on CPS directly affect on a physical environment, so learning techniques under training data attacks should be scrutinized. 
Toward the goal of resilient machine learning, we propose a resilience metric for the analysis and design of a resilient classification algorithm under training data attacks.
Traditional algorithms, such as convex loss linear classification algorithms and the \zo loss linear classification algorithm, are proved to be resilient under restricted conditions. However, the proposed \zo loss linear classification with a majority constraint is more resilient than others, and it is the maximally resilient algorithm among linear classification algorithms. The worst-case resilience bound of the proposed algorithm is then provided, suggesting how resilient the algorithm is under training data attacks. 

\textbf{Countermeasures.}
The resilience analysis on different linear classification algorithms provides us clues for countermeasures on training data attacks. Here, we briefly discuss a possible direction for countermeasures and its challenges. 
In general, additional algorithms can be considered to eliminate the worst-case situations in the analysis of each classification algorithm. 
For example, to defend against the point attack on SVMs, it might be considered to add a preprocessing step that saturates large values in training data. Specifically, if a designer knows the minimum and maximum range of features, then range can saturate the large values that contribute to the point attack. However, this might not be an effective countermeasure since the range of features is not known in general and the point attack can be conducted after the preprocessing step, not before it.

Here, we emphasize that our analysis, which is purposely focused on a classification algorithm exclusively, helps to devise countermeasures: combining a classification algorithm with a preprocessing step or using a complex classification algorithm (\eg hierarchical approach and neural networks). We believe that the advanced algorithms work better under the training data attack in general and our analysis on the simple algorithms (\eg SVMs) can be a building block for analyzing and devising advanced algorithms.

\textbf{Future works.}
As a future work, the following issues are worth being considered.
A more practical mixed integer linear program can speed up computational time (\eg \cite{nguyen2013algorithms}) and 
it would be promising to design and analyze multiple algorithms in tandem (one to monitor the data, one to learn a classifier).
It is also worth incorporating bounded noise error and designing error on $\Hs$ in analysis, and
extending to non-linear and multiclass classification problem.
%
Finally, to devise countermeasures, it would be promising to consider resilient algorithms that estimate attacker capabilities or model prior knowledge on attackers.




\vspace{-1.5ex}
\begin{acks}
\vspace{-0.5ex}
{This work was supported in part by NSF CNS-1505799 and the Intel-NSF Partnership for Cyber-Physical Systems Security and Privacy, by ONR N00014-17-1-2012, and
by Global Research Laboratory Program (2013K1A1A2A0207\-8326) through NRF, and the DGIST Research and Development Program (CPS Global Center) funded by the Ministry of Science, ICT \& Future Planning.  This material is also based in part on research sponsored by DARPA under agreement number FA8750-12-2-0247. The U.S. Government is authorized to reproduce and distribute reprints for Governmental purposes notwithstanding any copyright notation thereon. The views and conclusions contained herein are those of the authors and should not be interpreted as necessarily representing the official policies or endorsements, either expressed or implied, of DARPA or the U.S. Government.}
\end{acks}

\vspace{-1.5ex}
\balance
\bibliographystyle{ACM-Reference-Format}
\bibliography{../externals/mybibtex/rml} 

\appendix
\section*{Appendices}

%
%
%

\section{Connection to Statistical Learning Theory} \label{sec:connection_stl}

The goal of machine learning should minimize the generalization error; moreover, we believe the proposed resilience metric and generalization error are connected. To illustrate this, first, let us introduce a few standard notations from machine learning community. 
Let 
$\Rh(h) = \frac{1}{|\Dh|} \sum_{i=1}^{|\Dh|} \ell(y_{i}, h(\x_{i}))$ be the empirical risk of a classifier, 
$R(h) = \Exp_{(\x, y)} \ell(y, h(\x))$ be the expected risk of a classifier, 
$\Rh(\hh) = \min_{h \in \Hs} \Rh(h)$ be the empirical risk of a trained classifier,
$R(\hh)$ be the risk of the trained classifier,
$R(h_{*, \Hs}) = \inf_{h \in \Hs} R(h)$ be the expected risk of the best classifier in $\Hs$, and
$R(h_{*}) = \inf_{h \in \Ys^{\Xs}} \allowbreak R(h)$ be the Bayes risk.
Roughly speaking, the goal of learning problem is minimizing the \emph{generalization error}
\eqas{
R(\hh) - R(h_{*}).
}
The generalization error usually decomposed into two errors, called the \emph{estimation error} and \emph{approximation error}, as follows:
{\small
\eqas{
	\underbrace{R(\hh) - R(h_{*})}_{\text{generalization error}} 
	&= \underbrace{\[ R(\hh) - R(h_{*, \Hs}) \]}_{\text{estimation error}} + \underbrace{\[ R(h_{*, \Hs}) - R(h_{*}) \]}_{\text{approximation error}} 
}
}
Minimizing the approximation error is related to finding a ``good'' hypothesis space, $\Hs$. We acknowledge there are many techniques for identifying a hypothesis space including structural risk minimization, finding a regularization parameter using cross-validation, model selection, drop-out in deep learning, and manual modeling of neural network architecture. However, in our paper, we've assumed we have a ``good'' $\Hs$, meaning $\Hs$ contains the true classifier $f_{*}$ so approximation error always equals zero. Furthermore, since we've assumed there is no noise (or uncertainty) so Bayes risk $R(h_{*})$ is also zero -- which implies that $R(h_{*,\Hs}) = 0$. 

If we assume empirical risk minimization as a learning algorithm, the estimation error can be further decomposed as follows:
{\small
\eqas{
	\underbrace{R(\hh) - R(h_{*, \Hs})}_{\text{estimation error}} 
	&= \underbrace{\[ R(\hh) - \Rh(\hh) \]}_{A} + \underbrace{ \[ \Rh(\hh) - R(h_{*, \Hs}) \]}_{B}.
}
}
In statistical learning theory, it is known that the first term (\ie $A$ \footnote{some literatures call this term generalization error.}) and the second term (\ie $B$) converge in probability to zero due to Vapnik-Chervonenkis (VC) theory. 
Thus, if the number of training data goes to infinity and if empirical risk minimization is used as a learning algorithm, the estimation error converges in probability to zero. 
Furthermore, this implies that $\Rh(\hh) = 0$ converges to zero since (as described above) $R(h_{*,\Hs}) = 0$ by assumption.

We can consider the resilient learning problem in a similar fashion; the resilient learning problem is minimizing the estimation error
\eqas{
	R(\hh_{\alpha}) - R(h_{*, \Hs}), 
}
where $\hh_{\alpha}$ is a classifier trained over tampered training data. Similar to the traditional approach, this error can be decomposed as:
{\small
\eqas{
	\underbrace{R(\hh_{\alpha}) - R(h_{*, \Hs})}_{\text{estimation error}} 
		&= \[ R(\hh_{\alpha}) - \Rh(\hh) \] + \underbrace{ \[ \Rh(\hh) - R(h_{*, \Hs}) \]}_{B} \\
		= \underbrace{\[ R(\hh_{\alpha}) - \Rh(\hh_{\alpha}) \]}_{A'} &+ \underbrace{\[ \Rh(\hh_{\alpha}) - \Rh(\hh) \]}_{C} + \underbrace{ \[ \Rh(\hh) - R(h_{*, \Hs}) \]}_{B}
}
}
As before, the term $B$ converges in probability to zero due to VC theory and $R(h_{*, \Hs}) = 0$, which implies $\Rh(\hh) = 0$. But, we believe the behaviors of the term $A'$ and $C$ are uncharacterized and in this work, we consider the term $\Rh(\hh_{\alpha})$. Note that based on the notations in paper, 
\eqas{
	\Rh(\hh_{\alpha})  = \| \mathbf{\Rh}_{\ell_{01}} ( P_{\Hs, \ell}(\Dh_{\alpha}) | \Dh, \w ) \|_{1},
}
where $\w = \mat{cc}{|\Dh^{+}| & |\Dh^{-}|}^{\top}$, suggesting the resilience metric Eq. (3) is connected to existing statistical learning theory. Note that we've used $\infty$-norm and $\w = \mat{cc}{1 & 1}^{\top}$ instead since our setup also is used in literatures and we think it describes the behavior of the empirical risk of a trained classifier better when an attack is severe (in other words, $\alpha$ is large).

\section{Proofs on Lemmas, Propositions, and Theorems}


In this section, we describe the proofs of propositions and theorems stated in the paper, including supporting lemmas. 
For the notational simplicity, we use the following shorthands.
Let $N^+ = |\Dhp|$,  $N^- = |\Dhn|$, and $N = (N^+, N^-)$.
Let $\Rb_{\ell}(h | \Dh) = \sum_{i=1}^{|\Dh|} \ell (y_i, h(\x_i))$ such that  $\Rb_{\ell}(h | \Dh) = |\Dh| \cdot \Rh_{\ell}(h | \Dh)$.
When a \zo loss function is used, a set of classifiers has a same empirical risk. We denote, given $\Dh$, the class of classifiers that has a same empirical risk with $h$ as $\Ls_h = \left\{ h_i  \in \Ls \middle| \Rh_{\lzo}(h | \Dh) = \Rh_{\lzo}(h_i | \Dh) \right\}$. In this case, only the representative, $h$, is considered so that, for example, if there are three different classes of classifiers, they are simply called three different classifiers.
Figures used in proofs represent training data, where the point means the overlap of the annotated number of feature vectors. The positive and negative feature vectors are color-encoded in blue and red, respectively.
Set notations in the papers can be generalized to multiset notations. In the following proof, we assume all sets defined are multisets to handle duplicated training data pairs, which can happen if $\Xs \subseteq \integernum$. 
\subsection{Lemmas}
\begin{lemma} \label{lem:max_risk_gap} \label{lem:max_risk_gap_under_bad_attacks} \label{lem:max_risk_gap_under_good_attacks}
For all $\Dh$, $\alpha$, $\Dha$, and $h$,
the following is true:
\eqas{
\left| \Rh_{\ell_{01}}(h | \Dhap) - \Rh_{\ell_{01}}(h | \Dhp) \right| &\leq \frac{\alpha^+}{|\Dhp|} \text{~and~} \\
\left| \Rh_{\ell_{01}}(h | \Dhan) - \Rh_{\ell_{01}}(h | \Dhn) \right| &\leq \frac{\alpha^-}{|\Dhn|}.
}
\end{lemma}
\begin{proof}
First, for any classifier, $h$, $\Rh_{\ell_{01}}(h | \Dhap)$ can increase compared to $\Rh_{\ell_{01}}(h | \Dhp)$ if an attacker maliciously manipulates training data. 
Also, the amount of the empirical risk increased is at most $\frac{\alpha^+}{|\Dhp|}$ depending on the choice of the attacker. Formally,
\begin{align*}
	\Rh_{\ell_{01}}(h | \Dhap) - \Rh_{\ell_{01}}(h | \Dhp)
	&\leq  s \frac{\alpha^+}{|\Dhp|}
	\leq \frac{\alpha^+}{|\Dhp|},
\end{align*}
where $0 \leq s \leq 1$. Likewise, $\Rh_{\ell_{01}}(h | \Dhan) -  \Rh_{\ell_{01}}(h | \Dhn) \leq  \frac{\alpha^-}{|\Dhn|}$.

Next, for any classifier, $h$, $\Rh_{\ell_{01}}(h | \Dhap)$ can decrease compared to $\Rh_{\ell_{01}}(h | \Dhp)$ if an attacker helps $h$ for classification. 
Also, the amount of the empirical risk decreased is at most $\frac{\alpha^+}{|\Dhp|}$ depending on the choice of the attacker. Formally,
\begin{align*}
	\Rh_{\ell_{01}}(h | \Dhp) - \Rh_{\ell_{01}}(h | \Dhap) 
	&\leq  s \frac{\alpha^+}{|\Dhp|}	
	\leq \frac{\alpha^+}{|\Dhp|},
\end{align*}
where $0 \leq s \leq 1$. Likewise, $\Rh_{\ell_{01}}(h | \Dhn) - \Rh_{\ell_{01}}(h | \Dhan)  \leq \frac{\alpha^-}{|\Dhp|}$.
\end{proof}

\begin{lemma} \label{lem:conv_loss_lower_bound}
Let $\ell_c(\cdot)$ be a convex loss function and $h$ be linear. For all $k$, $y$, and $h$, there exists $\x$ such that $\ell_c(y, h(\x)) \geq k$.
\end{lemma}
\begin{proof}
Let $\ell_c(y, h(\x)) = \phi(t)$, where $t=yh(\x)$,  $\phi$ is convex, and $\phi(0) = 1$, which is a usual setup for convex losses \cite{bartlett2006convexity}.

By the convexity of $\phi$ and $\phi(0) = 1$,
\eqas{
	\frac{1}{2} \left( \phi(t) + \phi(0) \right) &\geq \phi\left(\frac{t}{2}\right) \\
	\phi(t) &\geq 2\phi\left(\frac{t}{2}\right) - 1.
}
Given $k$, $y$, and $h$, $\x$ can be chosen such that $2\phi\left(\frac{t}{2}\right) - 1 = 2 \phi \left( \frac{yh(\x)}{2} \right) - 1 \geq k$ since $yh(\x)$ can increase arbitrarily by moving $\x$ toward the direction of the normal vector of $h$ if $y = +1$ or the opposite direction of the normal vector of $h$ if $y = -1$.
This implies $\ell_c(y, h(\x)) = \phi(t) \geq k$ for some $\x$.
\end{proof}

\subsection{Proposition and Theorem Proofs}
\subsubsection{\Thm \ref{thm:maximal_resilience}} \label{sec:maximal_resilience_proof}
\begin{proof}
Let $\hha = P_{\Ls, \ell}(\Dha)$ for the notational simplicity of this proof, and
$\ell(y, h(\x)) = \phi(t)$, where $t = y h(\x)$.
Assume that $\phi$ is lower-bounded by $\mathbbm{1}\left\{ t \leq 0 \right\}$, it is a monotonically non-increasing function, and $\lim_{t \to \infty} \phi(t) = c$ for some scalar $c < 1$. Note that the assumptions are generalized from the assumptions on convex losses \cite{bartlett2006convexity} to cover non-convex ones.

\Thm \ref{thm:maximal_resilience} is formally stated as follows:
\begin{multline}
  \forall \ell \forall N \forall \alpha \left( 
  \alpha^+ \geq \frac{1}{2} N^+ \vee \alpha^- \geq \frac{1}{2} N^+
  \implies 
  \exists \Dh \exists \Dha \exists \hha, 
  \right. \\ \left.
  \Rh_{\ell_{01}}(\hha|\Dhp) = 1 \vee 
  \Rh_{\ell_{01}}(\hha|\Dhn) = 1
  \vphantom{\frac{1}{2}} \right). \label{eq:thm1_formal}
\end{multline}

Assume $\alpha^+ \geq \frac{1}{2} N^+$.
we partition the type of loss functions into two classes. 
One class is the set of the loss function that has a ``flat tail''. Formally, it is the set of $\phi$ such that $\forall s < t, \phi(s) = \phi(t) = d$ for some negative $t$ and some scalar $d \geq 1$. Call this set  $\Phi_1$. 
The other class is the set of the loss function that has a ``not-flat tail''. Formally, it is the set of $\phi$ such that $\exists s < t, \phi(s) > \phi(t)$ for any negative $t$. Call this set $\Phi_2$. 
We claim that $\Dh$ and $\Dha$ illustrated in \figref{fig:01_necessary_proof_case_1} makes $\Rh_{\ell_{01}}(\hha|\Dhp) = 1$ or $\Rh_{\ell_{01}}(\hha|\Dhn) = 1$ for any $\ell \in \Phi_1 \cup \Phi_2$. In \figref{fig:01_necessary_proof_case_1_b},  we denote feature vectors annotated by $\alpha^+$ as $\x$ and the point means a set, $\Ss$, of $\alpha^+$-pairs of $(\x, +1)$. Also, the three points are collinear.

First, assume $\ell \in \Phi_1$. As illustrated in \figref{fig:01_necessary_proof_case_1}, we consider two classes of classifiers, $\Ls_{h_1}$ and $\Ls_{h_2}$, where 
$h \in \Ls_{h_1}$ misclassifies all feature vectors annotated by $\alpha^+$ but correctly classifies others and 
$h \in \Ls_{h_2}$ misclassifies all feature vectors annotated by $N^+ - \alpha^+$ but correctly classifies others. 
By the definition of $\ell \in \Phi_1$, the followings are true:
\eqas{
\inf_{h \in \Ls_{h_1}} \Rb_\ell(h | \Dha) &= \lim_{\|{h}_1\|_2 \to 0}\Rb_\ell({h}_1 | \Dha) \\ & = d{\alpha^+} + c \left(N^+ - \alpha^+ + N^-\right), \\
\inf_{h \in \Ls_{h_2}} \Rb_\ell(h | \Dha) &= \lim_{\|{h}_2\|_2 \to 0}\Rb_\ell({h}_2 | \Dha) \\ & = d\left(N^+ - \alpha^+ \right) + c \left(\alpha^+  + N^-\right).
}

Also, from $d>c$ and $\alpha^+ \geq \frac{1}{2} N^+$,
\eqas{
	\inf_{h \in \Ls_{h_1}} \Rb_\ell(h | \Dha) &- \inf_{h \in \Ls_{h_2}} \Rb_\ell(h | \Dha) \\
	&= \left( d{\alpha^+} + c \left(N^+ - \alpha^+ + N^-\right) \right) - \\
	&\left( d\left(N^+ - \alpha^+ \right) + c \left(\alpha^+  + N^-\right) \right) \\
	&= d \left({2\alpha^+} - N^+\right) - c\left({2\alpha^+} -N^+\right) \\
	&= (d - c) \left({2\alpha^+} - N^+\right) \\
	&\geq 0.
}
This implies $\inf_{h \in \Ls_{h_1}} \Rb_\ell(h | \Dha) \geq \inf_{h \in \Ls_{h_2}} \Rb_\ell(h | \Dha) \geq \Rb_\ell(\hha | \Dha)$ and 
$\Ls_{\hha} \cap \Ls_{h_2} \neq \emptyset$. 
Thus, $\Rh_{\ell_{01}}(\hha|\Dhp) = 1$ for some $\hha$ and all $\ell \in \Phi_1$.

Next, assume $\ell \in \Phi_2$.
Same as before, as illustrated in \figref{fig:01_necessary_proof_case_1}, we consider two classes of classifiers, $\Ls_{h_1}$ and $\Ls_{h_2}$. By the definition of $\ell \in \Phi_2$, an attacker can choose $\x$ such that
\eqas{
\inf_{h \in \Ls_{h_1}} \sum_{(\x, y) \in \Ss} \ell(y, {h}(\x)) &\geq \inf_{h \in \Ls_{h_2}} \Rb_{\ell}(h | \Dha).
}
This is because the left term can be arbitrarily large if an attacker moves $\x$ to the right side in \figref{fig:01_necessary_proof_case_1_b}. However, the right term cannot be arbitrarily large since any $h \in \Ls_{h_2}$ correctly classifies $\x$.
Here, let $\bar{h} = \arg\inf_{h \in \Ls_{h_1}} \sum_{(\x, y) \in \Ss} \allowbreak \ell(y, h(\x))$.
Thus, the following is true:
{\small
\eqas{
\inf_{h \in \Ls_{h_1}} \!\!\!\! \Rb_{\ell}(h | \Dha) 
	&= \!\! \inf_{h \in \Ls_{h_1}} \!\! \left\{ \Rb_{\ell}(h | \Dha \setminus \Ss) + \sum_{(\x, y) \in \Ss} \ell(y, h(\x)) \right\} \\
	&\geq \!\! \inf_{h \in \Ls_{h_1}} \!\! \Rb_{\ell}(h | \Dha \setminus \Ss) + \sum_{(\x, y) \in \Ss} \ell(y, \bar{h}(\x)) \\
	&\geq \!\! \inf_{h \in \Ls_{h_1}} \!\! \Rb_{\ell}(h | \Dha \setminus \Ss) + \inf_{h \in \Ls_{h_2}} \Rb_{\ell}(h | \Dha) \\
	&\geq \!\! \inf_{h \in \Ls_{h_2}} \!\! \Rb_{\ell}(h | \Dha).
}
}
This implies $\Ls_{\hha} \cap \Ls_{h_2} \neq \emptyset$. Thus, $\Rh_{\ell_{01}}(\hha|\Dhp) = 1$ for some $\hha$ and all $\ell \in \Phi_2$.
Likewise, if $\alpha^- \geq \frac{1}{2} N^-$, then $\Rh_{\ell_{01}}(\hha|\Dhn) = 1$ for some $\hha$ and all $\ell$.
Therefore, \eqnref{eq:thm1_formal} is true.
\end{proof}

\subsubsection{\Thm \ref{thm:maximal_resilience_condition}} \label{sec:maximal_resilience_condition_proof}
\begin{proof}
Assume $\As_{P} = \bar{\Bs}^{c}$.
Since $\bar{\Bs}^c$ in \Thm \ref{thm:maximal_resilience} provides a theoretical upper bound of $\bar{\As}$, $\bar{\As} \subseteq \As_{P}$.
Also, by the definition of the maximal resilience attack condition, $\As_{P} \subseteq \bar{\As}$. Therefore, $\bar{\As} = \As_{P} = \bar{\Bs}^{c}$.
\end{proof}

%
%

{
\subsubsection{\Pro \ref{thm:feasibility_conv}} \label{sec:feasibility_conv_proof}
\begin{proof}
Let $\hha = P_{\Ls, \ell_c}(\Dha)$ for the notational simplicity of this proof. 
The proposition is formally stated as follows:
\begin{multline*}
  \forall \alpha \left( 
  \alpha^+ > 0 \vee \alpha^- > 0		
  \Longleftrightarrow 
  \exists \Dh \exists \Dha  \exists \hha, 
  \right. \\ \left.
  \Rh_{\ell_{01}}(\hha|\Dhp) = 1 \vee 
  \Rh_{\ell_{01}}(\hha|\Dhn) = 1 
  \vphantom{\frac{1}{2}} \right).
\end{multline*}

$(\Longleftarrow)$
We prove that if $\alpha^+ = 0$ and $\alpha^- = 0$, then $P_{\Ls, \ell_c}$ is not perfectly attackable.
Since an attacker cannot attack any point, $\Dha = \Dh$ and $P_{\Ls, \ell_c} (\Dha) = P_{\Ls, \ell_c} (\Dh)$. This implies $\Rh_{\lzo}(P_{\Ls, \ell_c} (\Dh) | \allowbreak \Dhp) \allowbreak = 0$ and $\Rh_{\lzo}(P_{\Ls, \ell_c} (\Dh) | \Dhn) = 0$, which implies $P_{\Ls, \ell_c}$ is not perfectly attackable.

$(\Longrightarrow)$
We prove that if $\alpha^+ > 0$ or $\alpha^- > 0$, then $P_{\Ls, \ell_c}$ is perfectly attackable.

\begin{figure} [tb!]
\centering
\begin{subfigure}[b]{0.49\linewidth}
\centering
\begin{tikzpicture}[scale=0.4]
  \tikzset{
        >=stealth',
        help lines/.style={dashed, thick},
        axis/.style={<->},
        important line/.style={thick},
        connection/.style={thick, dotted},
        }
    
  \path
  coordinate (zero) at (0, 0)
  coordinate (Dp) at (-1.5,0)
  coordinate (Dn) at (1.5, 0)
  coordinate (ep) at (-3, 2)
  coordinate (f1_1) at (0, 3)
  coordinate (f1_2) at (0, -3)
  coordinate (f1_3) at (-1.4, 0)
  coordinate (f1_4) at (1.4, 0)
  coordinate (f2_0) at (4, 0)
  coordinate (f2_1) at (4, 2)
  coordinate (f2_2) at (4, -2)
  coordinate (f2_3) at (-1+4, 0)
  coordinate (f2_4) at (1+4, 0)
  ;
  \filldraw [black] (Dp) circle (2pt) 
  node[above , blue] {{\scriptsize $N^+$}}
  ;
  \filldraw [black] (Dn) circle (2pt) 
  node[above , red] {{\scriptsize$N^-$}}
  ;
  
  \draw[draw=none, important line, color=black] (f1_1) -- (f1_2);
        
\end{tikzpicture}
\caption{$\Dh$}
\label{fig:conv_necessary_proof_case_1a}
\end{subfigure}
\begin{subfigure}[b]{0.49\linewidth}
\centering
\begin{tikzpicture}[scale=0.4]
  \tikzset{
        >=stealth',
        help lines/.style={dashed, thick},
        axis/.style={<->},
        important line/.style={thick},
        connection/.style={thick, dotted},
        }
    
  \path
  coordinate (zero) at (0, 0)
  coordinate (Dp) at (-1.5,0)
  coordinate (Dn) at (1.5, 0)
  coordinate (ep) at (5, 0)
  coordinate (f1_1) at (0, 3)
  coordinate (f1_2) at (0, -3)
  coordinate (f1_3) at (-1.4, 0)
  coordinate (f1_4) at (1.4, 0)
  coordinate (f2_0) at (0, 1)
  coordinate (f2_1) at (-3.5, 1)
  coordinate (f2_2) at (3.5, 1)
  coordinate (f2_3) at (-1.4+2, 0)
  coordinate (f2_4) at (0, -0.5)
  ;

  \filldraw [black] (Dp) circle (2pt) 
  node[above , blue] {{\scriptsize $N^+ - 1$}} 
  ;
  \filldraw [black] (Dn) circle (2pt) 
  node[above , red] {{\scriptsize $N^-$}}
  ;
  \filldraw [black] (ep) circle (2pt) 
  node[above , blue] {{\scriptsize 1}};
  ;

  \draw[important line, color=black] (f1_1) -- (f1_2);
  \draw[->, color=black] ($(zero)+(0, -2)$) -- ($(f1_3)+(0, -2)$) node[below] {{\scriptsize $h_1$}};
  \draw[->, color=black] ($(zero)+(0, -2)$) -- ($(f1_3)+(2.8, -2)$) node[below] {{\scriptsize $h_2$}};
  \draw[important line, color=black] ($(f1_1) + (2, 0)$) -- ($(f1_2) + (2, 0)$);
  \draw[->, color=black] ($(zero)+(0+2, -2+4)$) -- ($(f1_3)+(0+2, -2+4)$) node[below] {{\scriptsize $h_3$}};
  \draw[->, color=black] ($(zero)+(0+2, -2+4)$) -- ($(f1_3)+(2.8+2, -2+4)$) node[below] {{\scriptsize $h_4$}};
  \draw[draw=none, important line, color=black] (f1_1) -- (f1_2);
  \draw[important line, color=black] (f2_1) -- (f2_2);
  \draw[->, color=black] ($(f2_0)+(-3, 0)$) -- ($(f2_4)+(-3, 0)$) node[right] {{\scriptsize$h_5$}};
  \draw[->, color=black] ($(f2_0)+(-3, 0)$) -- ($(f2_4)+(-3, 2.8)$) node[right] {{\scriptsize$h_6$}};
          
%
        
\end{tikzpicture}
\caption{$\Dha$}
\label{fig:conv_necessary_proof_case_1b}
\end{subfigure}
\caption{$\Dh$ and $\Dha$ for the necessary condition proof of \Pro \ref{thm:feasibility_conv}.}
\label{fig:conv_necessary_proof_case_1}
\end{figure}

First, assume $\alpha^+ > 0$.
Consider $\Dh$ and $\Dha$ illustrated in \figref{fig:conv_necessary_proof_case_1}. 
An attacker converts $\Dh$ to $\Dha$ by moving one positive data pair, $(\x_j, +1)$, to the right side, such that three points are collinear as in \figref{fig:conv_necessary_proof_case_1b}. 

Given $\Dha$, there are six classes of classifiers, $\Ls_{h_i}$, where $i \in [1, 6]$ and $h_i$ are classifiers illustrated in \figref{fig:conv_necessary_proof_case_1b}.
By \Lem \ref{lem:conv_loss_lower_bound}, an attacker can choose $\x_j$ such that $\ell_c(+1, \bar{h}(\x_j)) \geq k$, where  
\eqas{
\bar{h} &= \arg\inf_{h \in \Ls_{h_1}} \ell_c(+1, h(\x_j)) \text{~and} \\
k &\geq \max_{i \in [2, 6]} \inf_{h \in \Ls_i} \Rb_{\ell_c}(h | \Dha).
}

From this,  the following is true:
{\small
\eqas{
\inf_{h \in \Ls_{h_1}} \!\!\!\! \Rb_{\ell_c}(h | \Dha) 
	&= \!\! \inf_{h \in \Ls_{h_1}} \!\! \left\{ \Rb_{\ell_c}(h | \Dha \setminus \{(\x_j, +1)\}) + \ell_c(+1, h(\x_j)) \right\} \\
	&\geq \!\! \inf_{h \in \Ls_{h_1}} \!\! \Rb_{\ell_c}(h | \Dha \setminus \{(\x_j, +1)\}) + \ell_c(+1, \bar{h}(\x_j)) \\
	&\geq \!\! \inf_{h \in \Ls_{h_1}} \!\! \Rb_{\ell_c}(h | \Dha \setminus \{(\x_j, +1)\}) + \inf_{h \in \Ls_{h_i}} \Rb_{\ell_c}(h | \Dha) \\
	&\geq \!\! \inf_{h \in \Ls_{h_i}} \!\! \Rb_{\ell_c}(h | \Dha), 
}
}
where $i \in [2, 6]$.
Thus, $\Ls_{\hha} \cap \left( \cup_{i=2}^6 \Ls_{h_i} \right) \neq \emptyset$.
This implies there exists a classifier that is trained over $\Dha$ can misclassify all positive or negative feature vectors in $\Dh$, \ie $\Rh_{\ell_{01}}(\hha| \allowbreak \Dhp) \allowbreak = 1$ or $\Rh_{\ell_{01}}(\hha|\Dhn)  = 1$ for some $\hha \in \Ls_{h_i}$.
Therefore, $P_{\Ls, \ell_c}$ is perfectly attackable.
\end{proof}

}

\subsubsection{\Pro \ref{thm:feasibility_01}} \label{sec:feasibility_01_proof}
\begin{proof}
Let $\hha = P_{\Ls, \ell_{01}}(\Dha)$.
\Pro \ref{thm:feasibility_01} is formally described as follows:
\begin{multline*} 
  \forall N \forall \alpha \left( 
  \alpha^+ \geq \frac{1}{2} N^+ \vee \alpha^- \geq \frac{1}{2} N^+ \vee
  \right. \\ \left.
  \alpha^+ + \alpha^- \geq N^- \vee \alpha^+ + \alpha^- \geq N^+ 
  \right. \\ \left.
  \implies 
  \exists \Dh \exists \Dha \exists \hha, 
  \Rh_{\ell_{01}}(\hha|\Dhp) = 1 \vee 
  \Rh_{\ell_{01}}(\hha|\Dhn) = 1
  \vphantom{\frac{1}{2}} \right).
\end{multline*}

$(\Longrightarrow)$
Prove that if the conditions, $\alpha^+ \geq \frac{1}{2} N^+$,  $\alpha^- \geq \frac{1}{2} N^+$, $\alpha^+ + \alpha^- \geq N^-$, or $\alpha^+ + \alpha^- \geq N^+$, satisfies, then $P_{\Ls, \ell_{01}}$ is perfectly attackable.
\begin{figure} [tb!]
\centering
\begin{subfigure}[b]{0.49\linewidth}
\centering
\begin{tikzpicture}[scale=0.4]
  \tikzset{
        >=stealth',
        help lines/.style={dashed, thick},
        axis/.style={<->},
        important line/.style={thick},
        connection/.style={thick, dotted},
        }
    
  \path
  coordinate (zero) at (0, 0)
  coordinate (Dp) at (-2,0)
  coordinate (Dn) at (1, 0)
  coordinate (ep) at (3, 0)
  coordinate (f1_1) at (0, 2)
  coordinate (f1_2) at (0, -2)
  coordinate (f1_3) at (-1.4, 0)
  coordinate (f1_4) at (1.4, 0)
  coordinate (f2_0) at (4, 0)
  coordinate (f2_1) at (4, 2)
  coordinate (f2_2) at (4, -2)
  coordinate (f2_3) at (-1+4, 0)
  coordinate (f2_4) at (1+4, 0)
  ;
  \filldraw [black] (Dp) circle (2pt) 
  node[above , blue] {{\scriptsize $N^+$}}
  ;
  \filldraw [black] (Dn) circle (2pt) 
  node[above , red] {{\scriptsize$N^-$}}
  ;
  
        \draw[draw=none, important line, color=black] (f1_1) -- (f1_2);
        
\end{tikzpicture}
\caption{$\Dh$}
\label{fig:01_necessary_proof_case_1_a}
\end{subfigure}
\begin{subfigure}[b]{0.49\linewidth}
\centering
\begin{tikzpicture}[scale=0.4]
  \tikzset{
        >=stealth',
        help lines/.style={dashed, thick},
        axis/.style={<->},
        important line/.style={thick},
        connection/.style={thick, dotted},
        }
    
  \path
  coordinate (zero) at (0, 0)
  coordinate (Dp) at (-2,0)
  coordinate (Dn) at (1, 0)
  coordinate (ep) at (3, 0)
  coordinate (f1_1) at (0, 2)
  coordinate (f1_2) at (0, -2)
  coordinate (f1_3) at (-1.4, 0)
  coordinate (f1_4) at (1.4, 0)
  coordinate (f2_0) at (2, 0)
  coordinate (f2_1) at (2, 2)
  coordinate (f2_2) at (2, -2)
  coordinate (f2_3) at (-1.4+2, 0)
  coordinate (f2_4) at (1.4+2, 0)
  ;

  \filldraw [black] (Dp) circle (2pt) 
  node[above , blue] {{\scriptsize $N^+ - \alpha^+$}} 
  ;
  \filldraw [black] (Dn) circle (2pt) 
  node[above , red] {{\scriptsize $N^-$}}
  ;
  \filldraw [black] (ep) circle (2pt) 
  node[above , blue] {{\scriptsize $\alpha^+$}};
  ;

        \draw[important line, color=black] (f1_1) -- (f1_2);
        \draw[->, color=black] ($(zero)+(0, -1)$) -- ($(f1_3)+(0, -1)$) node[below] {{\scriptsize $h_1$}};

        \draw[important line, color=black] (f2_1) -- (f2_2);
        \draw[->, color=black] ($(f2_0)+(0, -1)$) -- ($(f2_4)+(0, -1)$) node[below] {{\scriptsize$h_2$}};
%
        
\end{tikzpicture}
\caption{$\Dha$}
\label{fig:01_necessary_proof_case_1_b}
\end{subfigure}
\caption{$\Dh$ and $\Dha$ for the necessary condition proof of \Thm \ref{thm:maximal_resilience} and \Pro \ref{thm:feasibility_01}.}
\label{fig:01_necessary_proof_case_1}
\end{figure}

First, assume $\alpha^+ \geq \frac{1}{2} N^+$.
To show $P_{\Ls, \ell_{01}}$ is perfectly attackable, it is enough to find $\Dh$ and $\Dha$ for all $N$ and $\alpha$ that make $P_{\Ls, \ell_{01}}$ perfectly attackable.
Consider $\Dh$ and $\Dha$ described in \figref{fig:01_necessary_proof_case_1}.
By a \tma, $\Dha$ can be represented as in \figref{fig:01_necessary_proof_case_1_b}. Note that the attacker does not move negative feature vectors even though it may have the ability to do that. 
Since $\alpha^+ \geq \frac{1}{2} N^+$, $\Rb(h_1 | \Dha) \geq \Rb(h_2 | \Dha) \geq \Rb(\hha | \Dha)$.
This means there exists $\hha$ that misclassifies all positive or negative feature vectors in $\Dh$, \ie $\Rh_{\ell_{01}}(\hha|\Dhp) = 1$ or $\Rh_{\ell_{01}}(\hha|\Dhn) = 1$ for some $\hha$.
Therefore, for all $N$ and $\alpha$, there exist $\Dh$ and $\Dha$ that make $P_{\Ls, \ell_{01}}$ perfectly attackable.
Likewise, if $\alpha^- \geq \frac{1}{2} N^-$, for all $N$ and $\alpha$, there exist $\Dh$ and $\Dha$ that make $P_{\Ls, \ell_{01}}$ perfectly attackable.


\begin{figure} [tb!]
\centering
\begin{subfigure}[b]{0.49\linewidth}
\centering
\begin{tikzpicture}[scale=0.4]
  \tikzset{
        >=stealth',
        help lines/.style={dashed, thick},
        axis/.style={<->},
        important line/.style={thick},
        connection/.style={thick, dotted},
        }
    
  \path
  coordinate (zero) at (0, 0)
  coordinate (Dp) at (-3,0)
  coordinate (Dn) at (3, 0)
  coordinate (ep) at (-3, 2)
  coordinate (f1_1) at (0, 3)
  coordinate (f1_2) at (0, -3)
  coordinate (f1_3) at (-1.4, 0)
  coordinate (f1_4) at (1.4, 0)
  coordinate (f2_0) at (4, 0)
  coordinate (f2_1) at (4, 2)
  coordinate (f2_2) at (4, -2)
  coordinate (f2_3) at (-1+4, 0)
  coordinate (f2_4) at (1+4, 0)
  ;
  \filldraw [black] (Dp) circle (2pt) 
  node[above , blue] {{\scriptsize $N^+$}}
  ;
  \filldraw [black] (Dn) circle (2pt) 
  node[above , red] {{\scriptsize$N^-$}}
  ;
  
        \draw[important line, color=black] (f1_1) -- (f1_2);
        \draw[->, color=black] ($(zero)+(0, -2)$) -- ($(f1_3)+(0, -2)$) node[below] {{\scriptsize $h$}};
        
\end{tikzpicture}
\caption{$\Dh$}
\label{fig:01_necessary_proof_case_3a}
\end{subfigure}
\begin{subfigure}[b]{0.49\linewidth}
\centering
\begin{tikzpicture}[scale=0.4]
  \tikzset{
        >=stealth',
        help lines/.style={dashed, thick},
        axis/.style={<->},
        important line/.style={thick},
        connection/.style={thick, dotted},
        }
    
  \path
  coordinate (zero) at (0, 0)
  coordinate (Dp) at (-3,0)
  coordinate (Dn) at (3, 0)
  coordinate (ep) at (-3, 2)
  coordinate (f1_1) at (0, 3)
  coordinate (f1_2) at (0, -3)
  coordinate (f1_3) at (-1.4, 0)
  coordinate (f1_4) at (1.4, 0)
  
  coordinate (f2_0) at (0, 1)
  coordinate (f2_1) at (-3.5, 1)
  coordinate (f2_2) at (3.5, 1)
  coordinate (f2_3) at (-1.4+2, 0)
  coordinate (f2_4) at (0, -0.5)
  
  coordinate (f3_0) at (2, -4.5/7 * 2 -0.7)
  coordinate (f3_1) at (-3.5, 1.5)
  coordinate (f3_2) at (3.5, -3)
  coordinate (f3_3) at ($(f3_0) + (-0.5, -7/4.5*0.5)$)
  coordinate (f3_4) at (0, -0.5)
  ;

  \filldraw [black] (Dp) circle (2pt) 
  node[above , blue] {{\scriptsize $N^+ - \alpha^+$}} 
  ;
  \filldraw [black] (Dn) circle (2pt) 
  node[above , red] {{\scriptsize $N^- - \alpha^-$}}
  node[below , blue] {{\scriptsize $\alpha^+$}}
  ;
  \filldraw [black] (ep) circle (2pt) 
  node[above , red] {{\scriptsize $\alpha^-$}};
  ;

	\draw[important line, color=black] (f1_1) -- (f1_2);
        \draw[->, color=black] ($(zero)+(0, -2)$) -- ($(f1_3)+(0, -2)$) node[below] {{\scriptsize $h_1$}};
        \draw[important line, color=black] (f3_1) -- (f3_2);
	\draw[->, color=black] ($(f3_0)+(0, 0)$) -- ($(f3_3)+(0, 0)$) node[below] {{\scriptsize $h_2$}};
        \draw[important line, color=black] (f2_1) -- (f2_2);
        \draw[->, color=black] ($(f2_0)+(1, 0)$) -- ($(f2_4)+(1, 0)$) node[right] {{\scriptsize$h_3$}};

%
        
\end{tikzpicture}
\caption{$\Dha$}
\label{fig:01_necessary_proof_case_3b}
\end{subfigure}
\caption{$\Dh$ and $\Dha$ for the necessary condition proof of \Pro \ref{thm:feasibility_01}.}
\label{fig:01_necessary_proof_case_3}
\end{figure}

Next, assume $\alpha^+ + \alpha^- \geq N^-$.
For all $N$ and $\alpha$, there exist $\Dh$ and $\Dha$ that make $P_{\Ls, \ell_{01}}$ perfectly attackable as illustrated in \figref{fig:01_necessary_proof_case_3}.
When $P_{\Ls, \ell_{01}}$ is trained over $\Dha$, one possible optimal classifier can be $h_3$ as represented in \figref{fig:01_necessary_proof_case_3b} among eight different classifiers except for $h_1$ and $h_2$. Since
$\Rb(h_1 | \Dha) = \alpha^+ + \alpha^- \geq \Rb(h_2 | \Dha) = \alpha^+ \geq N^- - \alpha^-$, and $\Rb(h_3 | \Dha) = N^- - \alpha^-$, the empirical risk of $h_3$ is as small as $h_1$ and $h_2$. Thus, any classifier, $h \in \Ls_{h3}$, has $\Rh_{\ell_{01}}(h | \Dhp) = 1$ or $\Rh_{\ell_{01}}(h | \Dhn) = 1$, which implies for all $N$ and $\alpha$ there exists $\Dh$ and $\Dha$ that make $P_{\Ls, \ell_{01}}$ perfectly attackable.
Likewise, for all $N$ and $\alpha$, if $\alpha^+ + \alpha^- \geq N^+$, then $P_{\Ls, \ell_{01}}$ is perfectly attackable for some $\Dh$ and $\Dha$.

Therefore, for all the above mentioned four cases, there exist $\Dh$ and $\Dha$ that make $P_{\Ls, \ell_{01}}$ perfectly attackable.
\end{proof}

\subsubsection{\Thm \ref{thm:feasibility_maj}} \label{sec:feasibility_maj_proof}
\begin{proof}
Let $\hha = P_{\Ms, \ell_{01}}(\Dha)$,  $h^* = P_{\Ms, \ell_{01}}(\Dh)$, where $\Rh(h^* | \Dh) = 0$, and
$\As\!\! = \!\!\left\{ (\alpha^+, \alpha^-) \middle| \alpha^+ < \frac{1}{2} N^+ \text{~and~}  \allowbreak \alpha^- < \frac{1}{2} N^- \!\! \right\}$
Formally, show the following statement:
\begin{multline*}
  \forall N \forall \alpha \left( 
  \alpha^+ < \frac{1}{2} N^+ \wedge \alpha^- < \frac{1}{2} N^-
  \Longleftrightarrow 
  \forall \Dh \forall \Dha,  
  \right. \\ \left.
  \left( \mathcal{M} \neq \emptyset \right) \wedge 
  \left( \forall \hha,~\Rh_{\ell_{01}}(\hha | \Dhp) < 1 \wedge \Rh_{\ell_{01}}(\hha | \Dhn) < 1 \right) \vphantom{\frac{1}{2}} 
  \vphantom{\frac{1}{2}} \right).
\end{multline*}

$(\Longrightarrow)$
First, prove that if the conditions, $\alpha^+ < \frac{1}{2} N^+$ and $\alpha^- < \frac{1}{2} N^-$, hold, then $P_{\Ms, \ell_{01}}$ is feasible and resilient to $\alpha$-\tma~for all $a \in \As$.
To prove $\Ms \neq \emptyset$, 
it is enough to show $\Ms$ includes at least one element for all $N$, $\alpha$, $\Dh$, and $\Dha$. 
Consider $h^*$. 
By \Lem \ref{lem:max_risk_gap_under_bad_attacks} and the conditions,
\begin{align*}
\Rh_{\ell_{01}}(h^* | \Dhap) 
	\leq \Rh_{\ell_{01}}(h^* | \Dhp) + \frac{\alpha^+}{N^+} 
	\leq \frac{\alpha^+}{N^+} < \frac{1}{2}.
\end{align*}
Likewise, $\Rh(h^* | \Dhan) < \frac{1}{2}$, suggesting $h^* \in \Ms$ for all $N$, $\alpha$, $\Dh$, and $\Dha$.

Since $\Ms \neq \emptyset$, prove that $P_{\Ms, \ell_{01}}$ is resilient.
By the definition of $\Ms$, 
\begin{align*}
\Rh(\hha | \Dhap) < \frac{1}{2} \wedge \Rh(\hha | \Dhan) < \frac{1}{2}
\end{align*}
for all $\Dh$, $\Dha$, and $\hha$.
Thus, 
\begin{align*}
\Rh_{\ell_{01}}(\hha | \Dhp)
\leq \Rh_{\ell_{01}}(\hha | \Dhap) + \frac{\alpha^+}{N^+}
< \frac{1}{2} + \frac{1}{2}  = 1
\end{align*}
by \Lem \ref{lem:max_risk_gap_under_good_attacks}, the majority constraint, and the conditions. Likewise, $\Rh(\hha | \Dhn) < 1$.
Therefore,  for all $N$, $\alpha$, $\Dh$, and $\Dha$, $P_{\Ms, \ell_{01}}$ is resilient to $\alpha$-\tma~if $\alpha \in \As$.

$(\Longleftarrow)$
Next, prove that if $P_{\Ms, \ell_{01}}$ is feasible and resilient, then the conditions, $\alpha^+ < \frac{1}{2} N^+ \wedge \alpha^- < \frac{1}{2} N^-$, hold. Equivalently,
\begin{multline*}
  \forall N \forall \alpha \left( 
  \alpha^+ \geq \frac{1}{2} N^+ \vee \alpha^- \geq \frac{1}{2} N^-
  \implies 
  \exists \Dh \exists \Dha, 
  \right. \\ \left.
  \left( \Ms = \emptyset \right) \vee
  \left( \exists \hha,~\Rh(\hha | \Dhp) = 1 \vee \Rh(\hha | \Dhn) = 1 \right) \vphantom{\frac{1}{2}} 
  \vphantom{\frac{1}{2}} \right).
\end{multline*}
To show this, it is enough to find $\Dh$ and $\Dha$ for all $N$ and $\alpha$ that make $P_{\Ms, \ell_{01}}$ infeasible or perfectly attackable.
Assume $P_{\Ms, \ell_{01}}$ is feasible and check whether it is perfectly attackable.
The proof is same as proving $P_{\Ls, \ell_{01}}$ is perfectly attackable if $\alpha^+ \geq \frac{1}{2} N^+$ or $\alpha^- \geq \frac{1}{2} N^-$ in \Sec \ref{sec:feasibility_01_proof}.
The main difference is the feasible set. Here, $\Ms$ is used instead of $\Ls$. 
But, we can use the same proof to prove $P_{\Ms, \ell_{01}}$ is perfectly attackable
since $h_2$ in \figref{fig:01_necessary_proof_case_1_b} still satisfy the majority constraint such that $h_2 \in \Ms$.

Therefore, there exist $\Dh$ and $\Dha$ that make $P_{\Ms, \ell_{01}}$ perfectly attackable if $\alpha^+ \geq \frac{1}{2} N^+$ or $\alpha^- \geq \frac{1}{2} N^-$.
\end{proof}

\subsubsection{\Thm \ref{thm:robustness_maj}} \label{sec:robustness_maj_proof}

\begin{proof}
Let $\hha = P_{\Ms, \ell_{01}}(\Dha)$,
$h^* = P_{\Ms, \ell_{01}}(\Dh)$, where $\Rh(h^* | \Dh) = 0$, and 
$\Rb(\cdot) = \Rb_{\lzo}(\cdot)$ for the notational simplicity of this proof. 

Show that the worst-case resilience of $P_{\Ms, \lzo}$ is tightly bounded.
Formally, 
\begin{multline}
  \forall N \forall \alpha \left( 
  \alpha^+  < \frac{1}{2} N^+ \wedge \alpha^-  < \frac{1}{2} N^- 
  \implies 
  \right. \\ \left.
  \forall \Dh \forall \Dha \forall \hha,
  V(P_{\Ms, \ell_{01}} | N, \alpha) \leq 
  \right. \\ \left.
  \max\left( 
  	\frac{\min\left( 2\alpha^+ + \alpha^-, \alpha^+ + \frac{N^+ - 1}{2}\right)}{N^+}, 
	\right.\right. \\ \left.\left.
	\frac{\min\left( \alpha^+ + 2\alpha^-, \alpha^- + \frac{N^- - 1}{2}\right)}{N^-} \right)
  \!\! \vphantom{\frac{1}{2}} \right). \label{eq:maj_robustness_bnd_formal}
\end{multline}
Note that the proof of \Thm \ref{thm:feasibility_maj} implies if $\alpha^+ < \frac{1}{2} N^+$ and $\alpha^- < \frac{1}{2} N^-$, then $P_{\Ms, \lzo}$ is feasible. 

(\textbf{bounded})~
First, we prove that the worst-case resilience of $P_{\Ms, \lzo}$ is bounded.
The optimality of $\hha$ with \Lem \ref{lem:max_risk_gap} implies the following:
\eqas{
	\Rb(\hha | \Dhap) + \Rb(\hha | \Dhan) &\leq \Rb(h^* | \Dhap) + \Rb(h^* | \Dhan) - k \\
	\Rb(\hha | \Dhap) &\leq \Rb(h^* | \Dhap) + \Rb(h^* | \Dhan) -k \\ &- \Rb(\hha | \Dhan) \\
	&\leq \Rb(h^* | \Dhap) + \Rb(h^* | \Dhan) -k \\
	&= \Rb(h^* | \Dhap) - \Rb(h^* | \Dhp) \\&+ \Rb(h^* | \Dhan) - \Rb(h^* | \Dhn) -k \\
	&\leq \alpha^+ + \alpha^- -k,
}
where $k$ is a slack variable to satisfy the majority constraint such that $\alpha^+ + \alpha^- - k \leq \frac{N^+ - 1}{2}$.

If $\alpha^+ + \alpha^- \leq \frac{N^+ - 1}{2}$, $k =0$ maximizes $\Rb(\hha | \Dhap)$. Thus,
\eqa{
\Rb( \hha | \Dhp ) 
&\leq \Rb( \hha | \Dhap) + \alpha^+ \nonumber \\
&\leq 2\alpha^+ + \alpha^-. \label{eq:bound_1}
}

If $\alpha^+ + \alpha^- > \frac{N^+ - 1}{2}$, $k$ can be the smallest non-zero scalar such that $\alpha^+ + \alpha^- - k \leq \frac{N^+ - 1}{2}$. Thus,
\eqa{
\Rb( \hha | \Dhp ) 
&\leq \Rb( \hha | \Dhap) + \alpha^+ \nonumber \\
&\leq \alpha^+ + \frac{N^+ - 1}{2}. \label{eq:bound_2}
}

From the above two cases, 
\eqa{
	\Rh( \hha | \Dhp ) \leq \frac{\min \left( 2\alpha^+ + \alpha^-, \alpha^+ + \frac{N^+ - 1}{2} \right)}{N^+}. \label{eq:final_bound_1}
}
Likewise, 
\eqa{
	\Rh( \hha | \Dhn ) \leq \frac{\min \left( \alpha^+ + 2\alpha^-, \alpha^- + \frac{N^- - 1}{2} \right)}{N^-}. \label{eq:final_bound_2}
}
Therefore, for all $N$, $\alpha$, $\Dh$, $\Dha$, and $\hha$, \eqnref{eq:maj_robustness_bnd_formal} is true.


\begin{figure} [t!]
\centering
\begin{subfigure}[c]{0.49\linewidth}
\centering
\captionsetup{justification=centering}
\begin{tikzpicture}[scale=0.4]
  \tikzset{
        >=stealth',
        help lines/.style={dashed, thick},
        axis/.style={<->},
        important line/.style={thick},
        connection/.style={thick, dotted},
        }
  \coordinate (y) at (0,5);
  \coordinate (x) at (5,0);
    
  \path
  coordinate (zero) at (0, 0)
  coordinate (alpha-p) at (-1,0)
  coordinate (r-alpha-p) at (-1, 1*2)
  coordinate (alpha-n) at (1,1*2)
  coordinate (r-alpha-n) at (1, -0*2)
  coordinate (f1_1) at (-4, 0.25*2*4 + 1)
  coordinate (f1_2) at (4, -0.25*2*4 + 1)
  coordinate (f1_3) at (1, 1/0.25/2)
  coordinate (f2_1) at (-5, 0.5*5)
  coordinate (f2_2) at (5, -0.5*5)
  coordinate (f2_3) at (-1.0*1, -1.0*1/0.5)
  coordinate (f0_1) at (0, 0.75*2*3)
  coordinate (f0_2) at (0, -0.5*2*3)
  coordinate (f0_3) at (-2, 0)
  ;
  \filldraw [black] (alpha-p) circle (2pt) 
  node[ left, blue] {{\scriptsize $2\alpha^+ + \alpha^- - k$}}
  ;
  \filldraw [blue] (r-alpha-p) circle (2pt) 
  node[left, yshift=-0.5ex, blue] {{\scriptsize$N^+ - 2\alpha^+ - \alpha^- + k$}}
  ;
  \filldraw [red] (r-alpha-n) circle (2pt) 
  node[right, yshift=-0.5ex, red] {{\scriptsize $N^-$}}
  ;
        \draw[important line, color=black] (f0_1) -- (f0_2);
        \draw[->, thick, color=black] (0, -2) -- ($(f0_3) + (0, -2)$) node [above left] {{\footnotesize $h^*$}};
\end{tikzpicture}
\caption{$\Dh$ and $h^*$}
\label{fig:equal_example_hbar}
\end{subfigure}
\centering
\begin{subfigure}[c]{0.49\linewidth}
\centering
\captionsetup{justification=centering}
\begin{tikzpicture}[scale=0.4]
  \tikzset{
        >=stealth',
        help lines/.style={dashed, thick},
        axis/.style={<->},
        important line/.style={thick},
        connection/.style={thick, dotted},
        }
  \coordinate (y) at (0,5);
  \coordinate (x) at (5,0);
    
  \path
  coordinate (zero) at (0, 0)
  coordinate (alpha-p) at (-1,0)
  coordinate (r-alpha-p) at (-1, 1*2)
  coordinate (alpha-n) at (1,1*2)
  coordinate (r-alpha-n) at (1, -0*2)
  coordinate (f1_1) at (-4, 0.25*2*4 + 1)
  coordinate (f1_2) at (4, -0.25*2*4 + 1)
  coordinate (f1_3) at (1, 1/0.25/2)
  coordinate (f0_1) at (0, 0.75*2*3)
  coordinate (f0_2) at (0, -0.5*2*3)
  coordinate (f0_3) at (-2, 0)
  ;
        \draw[important line, color=black] (f1_1) -- (f1_2);
        \draw[->, thick, color=black] (0, 1) -- ($(f1_3) + (0, 1)$) node [above] {{\footnotesize $\hha$}};
        \filldraw [black] (alpha-p) circle (2pt) 
        node[ left, blue] {{\scriptsize $2\alpha^+ + \alpha^- - k$}}
        ;
        \filldraw [blue] (r-alpha-p) circle (2pt) 
        node[left, yshift=-0.5ex, blue] {{\scriptsize$N^+ - 2\alpha^+ - \alpha^- + k$}}
        ;
        \filldraw [red] (r-alpha-n) circle (2pt) 
        node[right, yshift=-0.5ex, red] {{\scriptsize $N^-$}}
        ;
        \draw[draw=none] (f0_1) -- (f0_2);
\end{tikzpicture}
\caption{$\Dh$ and $\hha$}
\label{fig:equal_example_hhat}
\end{subfigure}
\caption{An example of $\Dh$ for proving on the worst-case tight resilience bound of \Thm \ref{thm:robustness_maj}.}
\label{fig:maj_equal_example}
\end{figure}

\begin{figure}[t!]
\centering
\centering
\begin{tikzpicture}[scale=0.75]
  \tikzset{
        >=stealth',
        help lines/.style={dashed, thick},
        axis/.style={<->},
        important line/.style={thick},
        connection/.style={thick, dotted},
        }
  \coordinate (y) at (0,5);
  \coordinate (x) at (5,0);
    
  \path
  coordinate (zero) at (0, 0)
  coordinate (alpha-p) at (-1,0)
  coordinate (r-alpha-p) at (-1, 1*2)
  coordinate (alpha-n-att) at (-3, 0)
  coordinate (alpha-n) at (1,1*2)
  coordinate (r-alpha-n) at (1, -0*2)
  coordinate (alpha-p-att) at (3, -0*2)
  coordinate (f1_1) at (-4, 0.25*2*4 + 1)
  coordinate (f1_2) at (4, -0.25*2*4 + 1)
  coordinate (f1_3) at (1, 1/0.25/2)
  ;
        \draw[connection] (alpha-p-att) -- (alpha-n-att);
        \draw[important line, color=black] (f1_1) -- (f1_2);
        \draw[->, thick, color=black] (0, 1) -- ($(f1_3) + (0, 1)$) node [above] {{\footnotesize $\Rb(\hha | \Dha) = \alpha^+ + \alpha^- - k$}};
        \filldraw [black, color=black] (alpha-p) circle (2pt) 
        node[ below, blue] {{\scriptsize $-\alpha^+$}}
        ;
        \filldraw [black, color=blue] (r-alpha-p) circle (2pt) 
        ;
        \filldraw [black, color=black] (alpha-n-att) circle (2pt) 
        node[ below, red] {{\scriptsize$\alpha^-$}}
        ;
        \filldraw [black, color=red] (r-alpha-n) circle (2pt) 
        node[ below, red] {{\scriptsize$-\alpha^-$}}
        ;
        \filldraw [black, color=black] (alpha-p-att) circle (2pt) 
        node[ below, blue] {{\scriptsize$\alpha^+$}}
        ;
\end{tikzpicture}
\caption{An example of $\Dha$ for proving on the worst-case tight resilience bound of \Thm \ref{thm:robustness_maj}.}
\label{fig:maj_equal_examples_attack}
\end{figure}

(\textbf{tight})~
Next, we prove that the worst-case resilience bound is tight.
To prove the tightness, we show
there exist $\Dh$, $\Dha$, and $\hha$ by which the equality of \eqnref{eq:maj_robustness_bnd_formal} holds for all $N$ and $\alpha$ if $h$ is linear, $\alpha^+ < \frac{1}{2} N^+$, and $\alpha^- < \frac{1}{2} N^-$. 

The examples of $\Dh$, $h^*$, and $\hha$ are illustrated on \figref{fig:maj_equal_example}.
\figref{fig:equal_example_hbar} shows $\Dh$ and the corresponding optimal classifier $h^*$. 
\figref{fig:equal_example_hhat} shows $\Dh$ and the projected optimal classifier $\hha$ trained over $\Dha$. 
\figref{fig:maj_equal_examples_attack} represents $\Dha$ from which $\hha$ is obtained. 
Note that the annotated number of positive and negative feature vectors are the difference from the original numbers in \figref{fig:maj_equal_example}; \ie the blue dot has $N^+ - 2\alpha^+ - \alpha^- + k$ number of positive feature vectors as originally denoted and the red dot has $N^- - \alpha^-$ number of negative feature vectors. 

If $\alpha^+ + \alpha^- \leq \frac{1}{2} (N^+ - 1)$, let $k = 0$. 
There are 16 different linear classifiers that separate five dots in \figref{fig:maj_equal_examples_attack}. Each classifier has an empirical risk, $\Rb(\cdot)$, among $\alpha^+ + \alpha^-$, $2\alpha^+ + \alpha^-$, $2\alpha^+ + 2\alpha^-$ or is infeasible due to the majority constraint. Thus, $\hha$ can be optimal. From this, $\Rb(\hha | \Dhp) = 2\alpha^+ + \alpha^-$ that satisfies the equality in \eqnref{eq:bound_1}. 
Likewise, we can find $\Dh$, $\Dha$, and $\hha$ such that $\Rb(\hha | \Dhn) = \alpha^+ + 2\alpha^-$.

If $\alpha^+ + \alpha^- > \frac{1}{2} (N^+ - 1)$, let $k = \alpha^+ + \alpha^- - \frac{1}{2}(N^+ - 1)$. 
There are 16 different linear classifiers that separate 5 dots in \figref{fig:maj_equal_examples_attack}. Each classifier has an empirical risk, $\Rb(\cdot)$, among $\alpha^+ + \alpha^- - k$, $\alpha^+ + \alpha^-$, $2\alpha^+ + \alpha^- - k$, $2\alpha^+ + 2\alpha^- - k$ or is infeasible due to the majority constraint. Thus, $\hha$ can be optimal. From this, $\Rb(\hha | \Dhp) = \alpha^+ + \frac{1}{2}(N^+ - 1)$ that satisfies the equality in \eqnref{eq:bound_2}. 
Likewise, we can find $\Dh$, $\Dha$, and $\hha$ such that $\Rb(\hha | \Dhn) = \alpha^- + \frac{1}{2}(N^- - 1)$.

Therefore, for all $N$ and $\alpha$ such that $\alpha^+ < \frac{1}{2} N^+$ and $\alpha^- < \frac{1}{2} N^-$, 
there exist $\Dh$, $\Dha$, and $\hha$ that make the equalities in \eqnref{eq:final_bound_1} and \eqnref{eq:final_bound_2} hold.
\end{proof}

\section{Proof on Mixed-Integer Linear Program for 0-1 Linear Classification} \label{sec:milp_proof}
In this section, we derive Mixed-Integer Linear Program (MILP) for 0-1 linear classification algorithm from the standard 0-1 linear classification problem.

\begin{proof}
The 0-1 linear classification problem is minimizing the misclassification error of a classifier over training data, or equivalently the sum of a 0-1 loss of a classifier over each training data pairs, according to 
\eqas{
	\min_{\h} \sum_{i} \mathbbm{1}\{ y_{i} \neq \sign\(\h^{\top} \x_{i} \)\}.
}

We introduce new binary variable $z_{i} \in \{0, 1\}$ such that $z_{i} = \mathbbm{1}\{ y_{i} \h^{\top} \x_{i} \leq 0 \}$.
Since $ \mathbbm{1}\{ y_{i} \neq \sign\(\h^{\top} \x_{i} \)\} =  \mathbbm{1}\{ y_{i} \h^{\top} \x_{i} \leq 0 \}$ assuming $\sign(0) = 0$, $\sign(x) = 1$ if $x > 0$, and $\sign(x) = -1$ if $x < 0$,
the equivalent problem (by introducing equality constraints) is as follows:
\eqas{
	\min_{\h, \z} &~ \sum_{i} z_{i} \\
	\text{s.t.} &~ \forall i,~ z_{i} = \mathbbm{1}\{ y_{i} \h^{\top} \x_{i} \leq 0 \}.
}

Next, we introduce new variables $e_{i} \in \realnum$ to convert the indicator function into a real-valued function. The equivalent problem is as follows:
\eqas{
	\min_{\h, \e, \z} &~ \sum_{i} z_{i} \\
	\text{s.t.} &~\forall i,~ e_{i} = \mathbbm{1}\{ y_{i} \h^{\top} \x_{i} \leq 0 \},~ -z_{i} \leq e_{i} \leq z_{i}.
}
This problem is equivalent to the previous one since the optimal values $z_{i}^{*}$ are identical.
Specifically, 
if there are $i$ such that $y_{i} \h^{\top} \x_{i} \leq 0$ for some $\h$, then $z_{i}^{*} = 1$ in the previous problem if and only if $e_{i} = 1$ and $z_{i}^{*} = 1$ in this problem. 
Likewise, 
if there are $i$ such that $y_{i} \h^{\top} \x_{i} > 0$ for some $\h$, then $z_{i}^{*} = 0$ in the previous problem if and only if $e_{i} = 0$ and $z_{i}^{*} = 0$ in this problem. 
Thus, $z^{*}_{i} = \bar{z}_{i}$ for all $i$.

To remove the indicator function, we introduce new slack variable $c_{i} \in \realnum$. Also, let $| y_{i} \h^{\top} \x_{i}| \leq \delta$, assuming $\delta \in \realnum$ is sufficiently large. The equivalent problem is as follows:
\eqas{
	\min_{\h, \e, \z, \mathbf{c}} &~ \sum_{i} z_{i} \\ 
	\text{s.t.} &~ \forall i,~ e_{i} = c_{i} - y_{i} \h^{\top} \x_{i}, \\
	&~ - (\delta + \epsilon) z_{i} \leq e_{i} \leq (\delta + \epsilon) z_{i}, ~c_{i} > 0,
}
where $\epsilon > 0$ is a sufficiently small value.
To prove the equivalence, check whether the optimal solution $z_{i}^{*}$ of the previous problem and this problem are same.
Let $z_{i}^{*}$ and $\bar{z}_{i}$ be the optimal solution of the previous and this problem, respectively. 
If there are $i$ such that $y_{i} \h^{\top} \x_{i} \leq 0$ for some $\h$, then $e_{i} = 1$ and $z_{i}^{*} = 1$ in the previous problem, and $c_{i} = \epsilon$, $e_{i} > \epsilon$, and $\bar{z}_{i} = 1$ are optimal solutions in this problem, resulting $z_{i}^{*} = \bar{z}_{i}$. 
If there are $i$ such that $y_{i} \h^{\top} \x_{i} > 0$ for some $\h$, $e_{i} = 1$ and $z_{i}^{*} = 1$ in the previous problem, and $c_{i} = y_{i} \h^{\top} \x_{i}$, $e_{i} = 0$, and $\bar{z}_{i} = 0$ are optimal solutions in this problem, resulting $z_{i}^{*} = \bar{z}_{i}$. Thus, $z^{*}_{i} = \bar{z}_{i}$ for all $i$.

To reduce the number of optimization parameters, we consider the following problem:
\eqas{
	\min_{\h, \e, \z, c} &~ \sum_{i} z_{i} \\
	\text{s.t.} &~ \forall i,~ e_{i} \geq c - y_{i} \h^{\top} \x_{i} , \\
	&~ - (\delta + \epsilon) z_{i} \leq e_{i} \leq (\delta + \epsilon) z_{i}, ~c > 0.
}
The previous problem and this problem are equivalent since the optimal solution on $z_{i}$ for both problems are same.  
Specifically, 
let $z_{i}^{*}$ and $\bar{z}_{i}$ are the optimal solution of the previous and this problem, respectively.
First, assume $z_{i}^{*} = 1$. This implies $e_{i} \neq 0$ and $y_{i} \h^{\top} \x_{i} \leq 0$. Since, in this problem, $y_{i} \h^{\top} \x_{i} \leq 0$ implies $e_{i} > 0$, $\bar{z}_{i} = 1$. 
Likewise, $\bar{z}_{i} = 1$ implies $z_{i}^{*} = 1$.
Next, assume $z_{i}^{*} = 0$. This implies $e_{i} = 0$ and $y_{i} \h^{\top} \x_{i} > 0$. In this problem, $y_{i} \h^{\top} \x_{i} > 0$ implies $e_{i} = 0$ or $e_{i} = 1$. Thus, $\bar{z}_{i} = 0$. 
Likewise, $\bar{z}_{i} = 0$ implies $z_{i}^{*} = 0$.
Therefore, $\bar{z}_{i} = z_{i}^{*}$.

Finally, if each constraints are divided by $c$, then we have the following MILP.
\eqas{
	\min_{\h, \e, \z, c} &~ \sum_{i} z_{i} \\
	\text{s.t.} &~ \forall i,~ \frac{e_{i}}{c} \geq 1 - y_{i} \frac{\h^{\top}}{c} \x_{i} , \\
	&~ - \frac{\delta + \epsilon}{c} z_{i} \leq \frac{e_{i}}{c} \leq \frac{\delta + \epsilon}{c} z_{i},
}
and if let $e'_{i} = \frac{e_{i}}{c}$, $\h' = \frac{\h}{c}$, and $\delta' = \frac{\delta + \epsilon}{c}$, the final MILP is obtained. 
\end{proof}


\end{document}